\documentclass{neus2025}

\title[PST for Differentiable Ternary Logic Gate Networks]{Polynomial Surrogate Training for Differentiable\\Ternary Logic Gate Networks}

\usepackage{newtxtext}
\usepackage{newtxmath}

\usepackage{amsmath}
\usepackage{amssymb}
\usepackage{graphicx}
\usepackage[dvipsnames]{xcolor}
\usepackage{booktabs}
\usepackage{hyperref}
\usepackage{algorithm}
\usepackage{algpseudocode}
\usepackage{caption}

\usepackage{multirow}
\usepackage{subcaption}

\usepackage{enumitem}
\usepackage{bm}

\definecolor{logicHeader}{gray}{0.3}
\definecolor{logicTrue}{HTML}{90EE90}
\definecolor{logicFalse}{HTML}{F08080}
\definecolor{logicText}{HTML}{FFFFFF}

\newcommand{\R}{\mathbb{R}}

\newcommand{\Var}{\mathrm{Var}}

\newcommand{\TRUE}{\textsc{True}}
\newcommand{\FALSE}{\textsc{False}}
\newcommand{\UNKNOWN}{\textsc{Unknown}}
\newcommand{\T}{\mathcal{T}}

\newcommand{\loss}{\mathcal{L}}

\newcommand{\clip}{\operatorname{clip}}

\newcommand{\scsmall}{\textsc{small}}
\newcommand{\scmedium}{\textsc{medium}}
\newcommand{\sclarge}{\textsc{large}}
\newcommand{\scdeeper}{\textsc{deeper}}
\newcommand{\scvlarge}{\textsc{vlarge}}
\newcommand{\schuge}{\textsc{huge}}

\newcommand{\pst}{\text{PST}}
\newcommand{\dlgn}{\text{DLGN}}
\newcommand{\tlgn}{\text{TLGN}}
\newcommand{\lmax}{\lambda_{\max}}

\makeatletter
\let\jmlr@orig@appendix\appendix
\renewcommand{\appendix}{%
  \jmlr@orig@appendix
  \setcounter{equation}{0}%
  \renewcommand{\theequation}{\thesection.\arabic{equation}}%
  \@addtoreset{equation}{section}%
}
\makeatother

\coltauthor{\Name{Sai Sandeep Damera} \Email{sdamera@umd.edu}\\ \Name{Ryan Matheu} \Email{rmatheu@umd.edu}\\ \Name{Aniruddh G. Puranic} \Email{puranic@umd.edu}\\ \Name{John S. Baras} \Email{baras@umd.edu}\\ \addr University of Maryland, College Park, MD, USA}

\begin{document}
  \maketitle

  \begin{abstract}
    
    \noindent Differentiable logic gate networks (DLGNs) learn compact, interpretable Boolean circuits via gradient-based training, but all existing variants are restricted to the 16 two-input binary gates.
    Extending DLGNs to Ternary Kleene $K_3$ logic and training DTLGNs where the \textsc{Unknown} state enables principled abstention under uncertainty is desirable. However, the support set of potential gates per neuron explodes to $19{,}683$, making the established softmax-over-gates training approach intractable.
    We introduce \emph{Polynomial Surrogate Training} (PST), which represents each ternary neuron as a degree-$(2,2)$ polynomial with 9 learnable coefficients (a $2{,}187\times$ parameter reduction) and prove that the gap between the trained network and its discretized logic circuit is bounded by a data-independent commitment loss that vanishes at convergence.
    Scaling experiments from 48K to 512K neurons on CIFAR-10 demonstrate that this hardening gap contracts with overparameterization. Ternary networks train $2$-$3\times$ faster than binary DLGNs and discover true ternary gates that are functionally diverse.
    On synthetic and tabular tasks we find that the \textsc{Unknown} output acts as a Bayes-optimal uncertainty proxy, enabling selective prediction in which ternary circuits surpass binary accuracy once low-confidence predictions are filtered.
    More broadly, PST establishes a general polynomial-surrogate methodology whose parameterization cost grows only quadratically with logic valence, opening the door to many-valued differentiable logic.

  \end{abstract}

  
  \section{Introduction}
\label{sec:introduction}

Logic gate networks replace conventional arithmetic neurons with compositions of discrete two-input logic gates, producing circuits that are inherently compact and interpretable \citep{petersen2022deep}. Differentiable logic gate networks (DLGNs) make these circuits trainable through gradient descent by introducing continuous relaxations of the gate selection process.
Since their introduction, DLGNs have been extended to convolutional architectures \citep{petersen2024convolutional}, recurrent variants \citep{buhrer2025recurrent}, and connection-optimized formulations \citep{mommen2025method}, demonstrating competitive accuracy at extreme parameter efficiency.

All existing DLGNs share a common design: each two-input neuron maintains a softmax distribution over a fixed set of $K=16$ Boolean gates, computes a weighted blend of their outputs during training, and selects a single gate via argmax at inference.
This \emph{softmax-over-gates} regime is viable precisely because the binary gate space is small enough to enumerate exhaustively.
Yet it creates a structural train-test gap: the soft output $\sum_k p_k\, g_k(a,b)$ disagrees with the hard output $g_{k^*}(a,b)$ unless the softmax probability concentrates entirely on one gate. This is a degenerate limit with vanishing gradients.
Concurrent work has proposed Gumbel-noise injection and straight-through estimation to mitigate this gap \citep{yousefi2025mind}, though the fundamental issue with training cost remains unaddressed.

\paragraph{Why Ternary Logic?}
Binary DLGNs inherit the fundamental limitation of classical Boolean logic: every signal is either \TRUE{} or \FALSE{}, with no capacity to express uncertainty or indeterminate outcomes.
Many tasks require exactly this capacity---classification under sensor dropout, temporal logic evaluation over finite observation windows, and medical diagnosis from incomplete records all demand \emph{principled abstention}, where a network can output ``not yet determined'' rather than committing to an unsupported decision.
Kleene's strong three-valued logic $K_3$ \citep{kleene1952introduction, kleene2002mathematical}, with truth values $\T = \{-1,0,+1\}$ representing \FALSE{}, \UNKNOWN{}, and \TRUE{}, provides the minimal extension of Boolean logic that natively expresses this uncertainty.
However, extending DLGNs to ternary logic reveals a fundamental limitation of the softmax-over-gates paradigm: a two-input ternary gate admits $3^{3^2} = 19{,}683$ possible truth tables, making the categorical distribution approach intractable.

\paragraph{Polynomial Surrogate Training.}
We introduce PST, which replaces the categorical parameterization entirely: instead of learning a distribution over gates, each neuron learns the coefficients of a low-degree polynomial that directly represents a function over $\T^2$.
For ternary logic, this is a degree-$(2,2)$ polynomial with exactly 9 coefficients, reducing parameters by $2{,}187\times$ compared to softmax-over-gates while covering the full 19,683-gate vocabulary.
PST is everywhere differentiable, requires no Gumbel tricks or softmax temperatures, 
and we prove that the per-neuron discretization error is bounded by a data-independent 
commitment loss (Theorem~\ref{thm:pst-gap}), providing a principled mechanism to 
narrow the train-test gap.

\paragraph{The hardening gap.}
While commitment regularization bounds per-neuron discretization error, activation 
quantization across layers creates network-level train-test mismatch.
Empirically, this gap contracts with overparameterization, falling from 14.1pp to 
3.7pp as scale increases from 96K to 512K neurons on CIFAR-10.

\paragraph{Fourier analysis on $\T^2$.}
We develop a Fourier-analytic framework on $\{-1,0,+1\}^n$ with an orthonormal basis appropriate for Kleene $K_3$ logic, including a quadratic term $\varphi_2(x) = x^2 - 2/3$ with no Boolean analogue that captures \UNKNOWN-sensitivity.
Fourier coefficients serve as spectral complexity measures for principled regularization and post-hoc analysis.

\begin{figure}[t]
    \centering
    \includegraphics[width=1\linewidth]{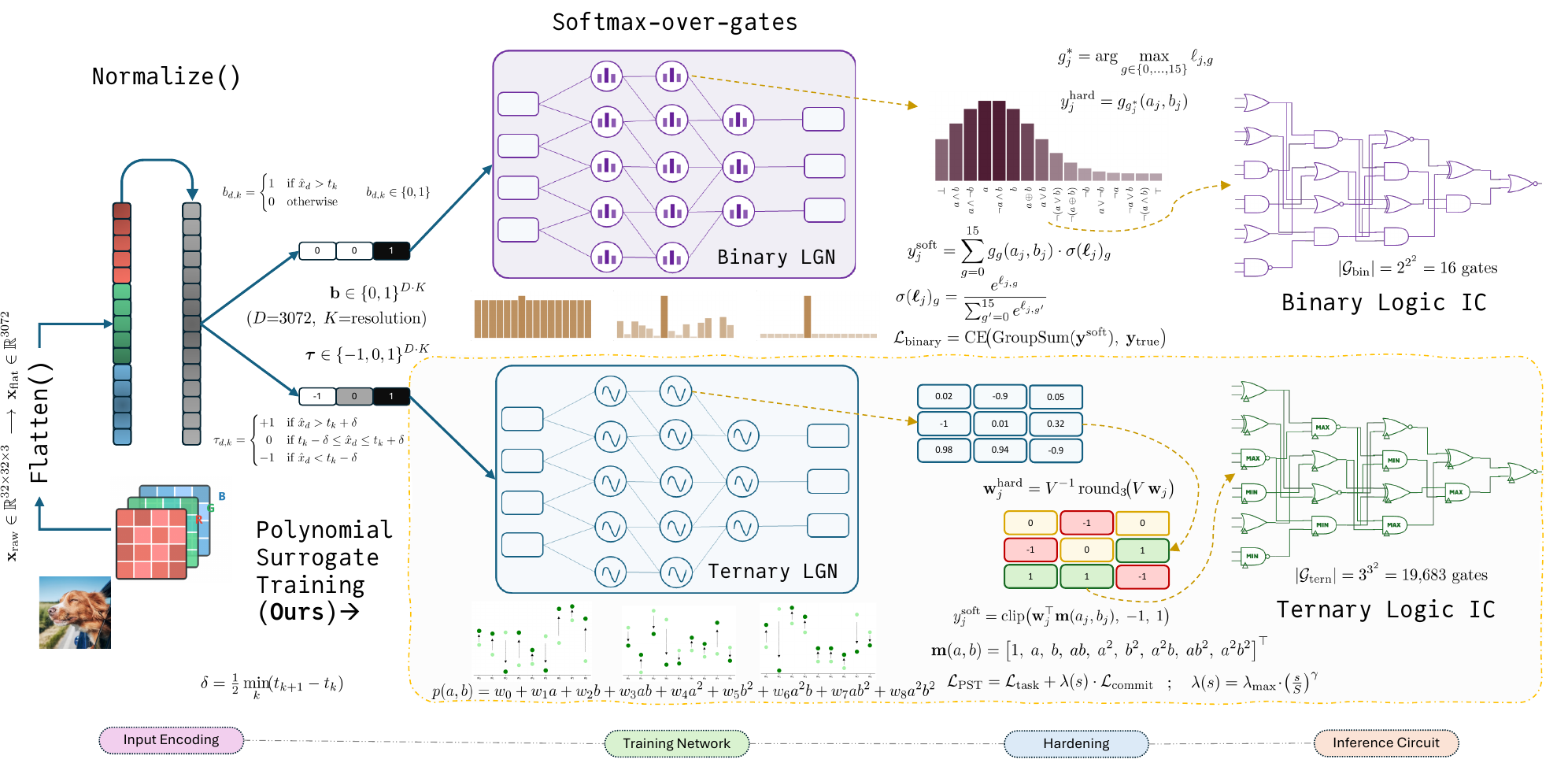}
    \caption{End-to-end comparison of binary DLGN training (top) and the proposed
    Polynomial Surrogate Training pipeline for ternary logic gate networks (bottom).
    \textbf{Input encoding} maps normalized features to bits/trits using temperature thresholding. 
    \textbf{Training:} binary DLGNs learn softmax distributions over 16~gates per
    neuron; PST instead learns 9~polynomial coefficients per neuron, parameterising
    the full space of $19{,}683$ ternary gates directly.
    \textbf{Hardening:} binary networks select the argmax gate; PST evaluates each
    polynomial on the $3{\times}3$ ternary grid, rounds to the nearest valid truth
    table, and recovers discrete gate coefficients via $\mathbf{w}_j^{\mathrm{hard}}
    = \mathbf{V}^{-1}\operatorname{round}_{\mathcal{T}}(\mathbf{V}\mathbf{w}_j)$.
    \textbf{Inference:} both pipelines produce logic circuits; which can be taped out as ultra-efficient ASICs for Inference}
    \label{fig:overview}
\end{figure}

\paragraph{Contributions.}
\begin{enumerate}[leftmargin=*,itemsep=2pt]
    \item We introduce \textbf{Polynomial Surrogate Training (PST)}, the first training regime for logic gate networks that parameterizes the function space directly (9 coefficients per neuron for ternary logic) rather than through a distribution over gates. PST is everywhere differentiable and provides \textbf{provable bounds on the per-neuron hardening gap} via a data-independent commitment loss (Theorem~\ref{thm:pst-gap}).
    \item We develop a \textbf{Fourier-analytic framework} on $\{-1,0,1\}^n$ with the orthogonal basis appropriate for Kleene $K_3$ logic, providing spectral characterization and principled regularization of learned gates.
    \item Scaling experiments from 48K to 512K neurons on CIFAR-10 demonstrate PST trains ternary circuits $2$-$3\times$ faster than binary DLGNs. On synthetic tasks, ternary circuits with \UNKNOWN{} outputs achieve selective prediction, surpassing binary accuracy when low-confidence predictions are filtered.
\end{enumerate}

  \section{Related Work}
\label{sec:related}

\paragraph{Efficient Neural Architectures.}
The computational demands of deep neural networks have motivated efficient model design \cite{bengio2013estimating}, including binary networks, sparse networks \citep{frankle2018lottery}, and lookup-table approaches \citep{zhu2016trained}.
DLGNs reduce every computation to a two-input logic gate lookup.
Weightless neural networks (WNNs) \citep{susskind2022weightless} use arbitrary lookup tables, whereas PST neurons are constrained to encode \emph{logic gates} with compositional semantics, enabling formal verification and spectral analysis.

\paragraph{Differentiable Logic Gate Networks}
\label{dt-lgn:subsec::background:::dlgn}

Training networks of discrete components like logic gates is challenging because they are non-differentiable. \cite{petersen2022deep} first introduced Differentiable Logic Gate Networks (DLGNs) to overcome this issue. Their approach relies relaxing discrete Boolean values $\{0,1\}$ to continuous $[0,1]$ and representing each of the 16 two-input gates as differentiable functions. Crucially, the learning does not involve selecting a single gate per neuron a priori. Instead, each neuron in a DLGN learns a categorical probability distribution over all 16 possible logic gates, computing a weighted blend during training and selecting the mode gate via argmax at inference to produce a discrete logic circuit. This final hardened network is a pure logic circuit, composed of fixed connections and discrete gates, making it extremely fast at inference and fully transparent. These networks can, in theory, be synthesized directly onto digital hardware for further speed, energy efficiency, and cybersecurity gains.

\section{Background and Preliminaries}
\label{dt-lgn:sec::background}
\subsection{Polynomial Representations for Boolean Logic Functions}
\label{dt-lgn:subsec::polynomial_representations}

The mathematical foundation of PST rests on the fact that logic functions over finite domains have exact polynomial representations.
Boolean functions admit exact real polynomial representations \citep{nisan1994degree, o2014analysis}. Concretely, a real polynomial $p:\mathbb{R}^n\to\mathbb{R}$ is said to represent a Boolean function $f:{\{-1,\,1\}}^n\to\{-1,\,1\}$ if for all $x\in{\{-1,\,1\}}^n$, $p(x)=f(x)$. Here we interpret \TRUE~as $+1$ and \FALSE~as $-1$. Because $z^{2k}=1,\,z^{2k+1}=z,\,k\in\mathbb{N}$ for all $z\in\{-1,\,1\}$, these polynomials are reduced to multi-linear polynomials. Let the set $P_n$ contain all real multi-linear polynomials that represent Boolean functions with $n$ inputs. Then $(P_n,\,{\cdot}\,)$ is a finite abelian group with group multiplication ``$\,{\cdot}\,$'' defined as polynomial multiplication. This subtle notion allows for the definition of a Fourier expansion on representing polynomials, the details of which are described in Appendix \ref{appendix:polynomial_algebra}.

We extend this to the balanced ternary system with truth constants $\T = \{-1, 0, +1\}$ representing \FALSE, \UNKNOWN, and \TRUE~respectively. Analogous to the Boolean case, a real multivariate polynomial $p$ is said to represent a ternary function of the form $f:{\T}^n\to\T$ if for all $x\in{\T}^n$, $p(x)=f(x)$. Consider $z\in\mathcal{T}$. For integer $k$ odd, $k\geq1$, $z^k=z$. For integer $k$ even, $k\geq2$, $z^k=z^2$. Therefore multivariate polynomials representing ternary functions are \textit{multi-quadratic}, containing variables with powers no greater than two.

\subsection{Fourier Analysis on $\T$}
\label{dt-lgn:subsec::fourier_basis}

The monomial basis, while convenient for parameterization, is not orthogonal under the uniform inner product $\langle f,\,g \rangle = \frac{1}{3^n}\sum_{x \in \T^n} f(x)\, g(x)$.
Consider the set of single input ($n=1$) ternary functions of the form $f:\mathcal{T}\to\mathcal{T}$. Applying Gram-Schmidt orthogonalization to the basis $\{1,\,x,\, x^2\}$ yields the univariate orthogonal basis:
\[
    \varphi_0(x) = 1, \qquad \varphi_1(x) = x, \qquad \varphi_2(x) = x^2 - \tfrac{2}{3}.
\]
The term $\varphi_2$ is the ``centered quadratic'': it measures whether $x$ is extreme ($\pm 1$, where $\varphi_2(x) = \tfrac{1}{3}$) versus neutral ($0$, where $\varphi_2(x) = -\tfrac{2}{3}$).
This basis function has no Boolean analogue---it captures the \UNKNOWN-sensitivity that distinguishes Kleene $K_3$ from classical logic.
The bivariate Fourier basis consists of the nine products $\{\varphi_i(x)\,\varphi_j(y)\}_{0 \leq i,j \leq 2}$, and the Fourier expansion of any $f: \T^2 \to \R$ is $f = \sum_{i,j} \hat{f}_{ij}\, \varphi_i \varphi_j$.
The Fourier coefficients $\hat{f}_{ij}$ serve as spectral complexity measures: the $L_1$ norm $\sum |\hat{f}_{ij}|$ quantifies gate complexity, and we use it as a regularizer during training to bias learning toward spectrally sparse (interpretable) gates.
The full derivation is provided in Appendix~\ref{appendix:spectral}.

  \section{The Polynomial Surrogate Training Framework for DTLGNs}
\label{dt-lgn:sec::formulation}
\subsection{PST Formulation}
\label{dt-lgn:subsec::pst_formulation}

The softmax-over-gates training regime creates an irreducible train-test gap: the soft network computes convex combinations of gate outputs while the hard network executes a single discrete gate.
Gumbel relaxation approaches \citep{kim2023deep} and Straight Through Estimation techniques \citep{yousefi2025mind} reduce but do not eliminate this gap, and are intractable for ternary logic's 19,683-gate support.

Our proposed Polynomial Surrogate Training (PST) regime addresses this gap by replacing the categorical parameterization with a direct polynomial parameterization. Each neuron learns a 9-coefficient polynomial $p_{\mathbf{w}}: \mathbb{R}^2 \rightarrow \mathbb{R}$ that approximates a ternary gate truth table on $\mathcal{T}^2$, with per-neuron discretization error bounded by a commitment regularizer.

\paragraph{PST neuron.}
Each PST neuron represents a two-input ternary function as a degree-$(2,2)$ polynomial with 9 learnable coefficients $\mathbf{w} \in \R^9$:
\begin{equation}
\label{eq:pst-neuron}
p_{\mathbf{w}}(a, b) = \mathbf{w}^\top \mathbf{m}(a, b),
\end{equation}
where $\mathbf{m}(a,b) = [1,\, a,\, b,\, ab,\, a^2,\, b^2,\, a^2 b,\, ab^2,\, a^2 b^2]^\top$ is the monomial basis.
Since the $9{\times}9$ Vandermonde matrix $\mathbf{V}$ evaluating this polynomial on $\T^2$ is invertible, the parameterization is universal: any function $f: \T^2 \to \R$ has a unique representation.
The polynomial is $C^\infty$-smooth and linear in $\mathbf{w}$, requiring no softmax or Gumbel relaxations, and evaluates in 8 multiplications and 8 additions per neuron.

\paragraph{Network architecture.}
A PST-DTLGN is specified by its depth $L$, widths $\{n_l\}_{l=0}^L$, a connectivity map $\mathbf{C}$ assigning two parent neurons to each neuron, and polynomial coefficients $\mathbf{w}_j^{(l)} \in \R^9$.
The training forward pass applies a clip nonlinearity after each polynomial evaluation:
\begin{equation}
\label{eq:forward-train}
h_j^{(l)} = \clip\!\Big(p_{\mathbf{w}_j^{(l)}}\!\big(h_{s_j}^{(\cdot)},\, h_{t_j}^{(\cdot)}\big)\Big),
\end{equation}
where $\clip(x) = \max(-1, \min(1, x))$.
Clipping preserves the range $[-1,1]$, prevents polynomial blowup across layers, preserves grid points $\{-1,0,+1\}$ exactly, and implicitly regularizes toward ternary-range outputs through its zero-gradient saturation regime.

\paragraph{Initialization.}
Weights are drawn i.i.d.\ from $\mathcal{N}(0, 0.45^2)$, chosen to ensure $\Var[p_{\mathbf{w}}(a,b)] \approx 1$ on ternary inputs, analogous to Xavier initialization \citep{pmlr-v9-glorot10a}.

\subsection{Training}
\label{dt-lgn:subsec::pst_training}

The training objective combines task loss with a commitment regularizer:
\begin{equation}
\label{eq:total-loss}
\loss(\mathbf{W}) = \loss_{\mathrm{task}}(\mathbf{W}) + \lambda(t) \cdot \mathcal{R}_A(\mathbf{W}),
\end{equation}
where $\mathbf{W} = \{\mathbf{w}_j^{(l)}\}$ is the set of all polynomial coefficients and the commitment loss is:
\begin{equation}
\label{eq:commitment-loss}
\mathcal{R}_A(\mathbf{W}) = \frac{1}{N}\sum_j \frac{1}{q^2}\sum_{(a,b) \in \mathcal{Q}^2} \operatorname{dist}\!\big(p_{\mathbf{w}_j}(a,b),\, \mathcal{Q}\big)^2,
\end{equation}
with $\operatorname{dist}(x, \mathcal{Q}) = \min_{v \in \mathcal{Q}} |x - v|$, where $\mathcal{Q}$ is the set of $q$ truth values.
For ternary logic, $q=3$ and $\mathcal{Q} = \{-1,0,1\}$.
The regularization weight $\lambda(t) = \lambda_{\max}(t/T)^\gamma$ is annealed from $\approx 0$ (free exploration) to $\lambda_{\max}$ (strong commitment) using a scheduling regime. 

\subsection{Hardening: From Polynomials to Gates}
\label{dt-lgn:subsec::pst_discretization}

At inference, each PST neuron is converted to a discrete ternary gate via Algorithm~\ref{alg:hardening}.
Since rounding always produces a valid truth table in $\T^9$, every trained PST neuron discretizes to one of the 19,683 possible two-input ternary gates without requiring a curated vocabulary---a fundamental advantage over the softmax regime, where the vocabulary must be fixed before training.

\begin{algorithm}[t]
\caption{PST Neuron Hardening}\label{alg:hardening}
\KwIn{Polynomial coefficients $\mathbf{w}_j \in \R^9$, Vandermonde matrix $\mathbf{V}$, gate library $\mathcal{G}$}
\KwOut{Discrete gate $g_j^*: \T^2 \to \T$}
$\mathbf{t}_j \gets \mathbf{V}\mathbf{w}_j$ \tcp*{Evaluate polynomial at grid points $\T^2$}
$\bar{\mathbf{t}}_j \gets \operatorname{round}_\T(\mathbf{t}_j)$ \tcp*{Round each entry to nearest value in $\T$}
$g_j^* \gets \mathcal{G}[\bar{\mathbf{t}}_j]$ \tcp*{Lookup gate in precomputed library}
\Return{$g_j^*$}
\end{algorithm}

\subsection{The Hardening Gap}
\label{dt-lgn:subsec::hardening-gap}

Hardening converts the continuous training-time model into a discrete circuit by rounding each neuron's polynomial to the nearest gate in the truth-table lattice $\Lambda_q = \mathcal{Q}^{q^2}$ (see Appendix~\ref{appendix:lattice} for lattice geometry).
For ternary logic, this lattice has $3^9 = 19{,}683$ points in $\R^9$.

\begin{theorem}[PST Hardening Gap Bound]
\label{thm:pst-gap}
Let $\mathcal{N}$ be a $q$-logic PST network with $N$ neurons.
Let $\mathbf{t}_j = [p_{\mathbf{w}_j}(a,b)]_{(a,b) \in \mathcal{Q}^2}$ be neuron $j$'s soft truth table and $\bar{\mathbf{t}}_j \in \Lambda_q$ its hardened truth table. Then
\[
    \mathcal{R}_A(\mathbf{W}) = \frac{1}{N}\sum_{j=1}^{N} \frac{1}{q^2}\|\mathbf{t}_j - \bar{\mathbf{t}}_j\|_2^2.
\]
\end{theorem}
\begin{proof}
By Algorithm~\ref{alg:hardening}, hardening rounds each polynomial evaluation to the nearest value in $\mathcal{Q}$, so $\bar{t}_{j,(a,b)} = \arg\min_{v \in \mathcal{Q}} |p_{\mathbf{w}_j}(a,b) - v|$.
Therefore $(p_{\mathbf{w}_j}(a,b) - \bar{t}_{j,(a,b)})^2 = \operatorname{dist}(p_{\mathbf{w}_j}(a,b), \mathcal{Q})^2$.
Summing over all grid points and averaging over neurons yields the result.
\end{proof}

\begin{remark}[Scaling and Data-Independence]
PST's design improves with logic valence: a $q$-logic neuron requires only $q^2$ coefficients (vs.\ softmax's super-exponential $q^{q^2}$ logits) while achieving $O(1/q)$ rounding tolerance.
The commitment loss $\mathcal{R}_A$ is data-independent, so Theorem~\ref{thm:pst-gap} holds for all inputs in $\mathcal{Q}^n$, not just the training distribution.
\end{remark}

  \section{Results}
\label{dt-lgn:sec::results}

We evaluate PST on two fronts: scaling experiments demonstrating that ternary networks train efficiently at CIFAR-10 scale (Section~\ref{dt-lgn:subsec::results-pst}), and analysis of principled abstention via UNKNOWN ($U$) outputs on synthetic tasks (Section~\ref{dt-lgn:subsec::results-ternary}).
All models use JAX \citep{jax2018github} with Equinox \citep{kidger2021equinox} and Adam.
Binary DLGNs follow \citet{petersen2022deep}. Full details in Appendix~\ref{appendix:expts_results}.

\subsection{Training DTLGNs at scale using PST}
\label{dt-lgn:subsec::results-pst}

We demonstrate that Polynomial Surrogate Training (\pst{}) successfully trains
differentiable ternary logic gate networks (\tlgn{}) at CIFAR-10 scale,
achieving soft accuracy parity with binary \dlgn{}s while training
2--3$\times$ faster. The hardening gap (the accuracy cost of converting
continuous polynomials to discrete ternary circuits) contracts from
14.1\,pp to 3.7\,pp as network scale increases from 96K to 512K neurons,
establishing a clear trajectory toward gap elimination through
overparameterization and commitment loss modulation.

\paragraph{Setup.}
CIFAR-10 (50K train / 10K test) is flattened to 3,072 dimensions and encoded
via resolution-4 balanced encoding ($K{=}3$ thresholds), producing 9,216 binary
or ternary input features. We evaluate six homogeneous 4-layer architectures
from \scsmall{} (4$\times$12K neurons, 48K total) to \schuge{}
(4$\times$128K, 512K total), plus a 5-layer \scdeeper{} variant (60K total);
see Table~\ref{tab:configs} for full configurations. All share identical
random sparse connectivity (seed$\,{=}\,$42) with two inputs per neuron.
Binary \dlgn{} uses Adam (lr$\,{=}\,$0.01), cross-entropy via GroupSum
($k{=}10$, $\tau{=}33.3$). Ternary \pst{} uses Adam (lr adapted per scale),
$\lmax{=}0.1$, $\gamma{=}2.0$ (quadratic lambda annealing), MSE via GroupSum.
Batch size 100 for both. Soft accuracy: full 10K test set; circuit accuracy:
2K samples ($\pm 2.2$\,pp Wilson CI at ${\sim}50\%$).
Table~\ref{tab:main} presents the central results across all scales.

\begin{table}[t]
  \centering
  \caption{\textbf{Performance across all scales.}
    Soft: continuous forward pass (10K test). Circuit: post-hardening discrete
    evaluation (2K samples; $\pm 2.2$\,pp CI). Gap = Soft $-$ Circuit.
    UNK\%: fraction of \tlgn{} output neurons producing zero
    (not sample-level abstention). Speed: \dlgn{}/\tlgn{} wall-time ratio
    (RTX~4090).}
  \label{tab:main}
  \small
  \begin{tabular}{@{}lrcccrr@{}}
    \toprule
    Model & Neurons & Soft Acc & Circuit Acc & Gap (pp) & UNK\% & Speed \\
    \midrule
    \dlgn{}-\scsmall{}    &  48K & 49.3\% & 48.9\% & $+0.5$  & ---    & \multirow{2}{*}{1.5$\times$} \\
    \tlgn{}-\scsmall{}    &  48K & 45.4\% & 42.6\% & $+2.8$  & 5.9\%  & \\
    \addlinespace
    \dlgn{}-\scmedium{}   &  96K & 50.7\% & 51.4\% & $-0.8$  & ---    & \multirow{2}{*}{2.0$\times$} \\
    \tlgn{}-\scmedium{}   &  96K & 50.3\% & 36.1\% & $+14.1$ & 6.2\%  & \\
    \addlinespace
    \dlgn{}-\sclarge{}    & 144K & 51.4\% & 51.7\% & $-0.4$  & ---    & \multirow{2}{*}{2.1$\times$} \\
    \tlgn{}-\sclarge{}    & 144K & 51.5\% & 38.5\% & $+13.0$ & 6.0\%  & \\
    \addlinespace
    \dlgn{}-\scvlarge{}   & 192K & 51.6\% & 52.0\% & $-0.4$  & ---    & \multirow{2}{*}{2.1$\times$} \\
    \tlgn{}-\scvlarge{}   & 192K & 51.2\% & 39.2\% & $+11.9$ & 9.3\%  & \\
    \addlinespace
    \dlgn{}-\schuge{}     & 512K & 52.5\% & 53.0\% & $-0.5$  & ---    & \multirow{2}{*}{3.1$\times$} \\
    \tlgn{}-\schuge{}     & 512K & 52.1\% & 48.4\% & $+3.7$  & 24.6\% & \\
    \addlinespace
    \dlgn{}-\scdeeper{}   &  60K & 49.3\% & 49.4\% & $-0.1$  & ---    & \multirow{2}{*}{1.8$\times$} \\
    \tlgn{}-\scdeeper{}   &  60K & 45.0\% & 43.0\% & $+2.0$  & 7.3\%  & \\
    \bottomrule
  \end{tabular}
\end{table}

\paragraph{PST trains ternary circuits at scale.}
Both architectures achieve comparable soft accuracy at all scales $\geq$
\scmedium{}, plateauing at ${\sim}52\%$, a ceiling imposed by the resolution-4
encoding rather than network capacity. At \schuge{}: \tlgn{} 52.1\% vs.\
\dlgn{} 52.5\% ($p{=}0.57$, not significant). This parity is the key
feasibility result: \pst{}'s free polynomial parameterization (9 coefficients per
neuron, searching 19,683 possible ternary gates) matches the representational
capacity of the constrained softmax-over-16-gates approach, confirming
that the ternary polynomial search space is trainable at scale.

\paragraph{The hardening gap closes with overparameterization.}
The \pst{} hardening gap follows a non-monotonic trajectory: modest at
\scsmall{} ($+2.8$\,pp), peaking at \scmedium{} ($+14.1$\,pp), then
contracting steadily through \sclarge{} ($+13.0$), \scvlarge{} ($+11.9$),
to \schuge{} ($+3.7$\,pp). This contraction is accompanied by a $10\times$
reduction in per-neuron hardening error (0.298 $\to$ 0.029) and a
$4\times$ increase in $U$ output neurons (5.9\% $\to$ 24.6\%). For the
neurons whose hardened truth tables map to zero, contributing no signal to the
GroupSum class scores (a possibility unique to ternary circuits, where $0$ is a
valid output value). This indicates that overparameterization enables implicit
pruning: poorly-committed neurons are mapped to zero-output gates during
hardening, while the surviving neurons are well-committed to valid ternary
truth tables.
The mechanism is detailed in \S\ref{sec:gap-mechanism};
Figure~\ref{fig:perclass-main} shows the per-class impact.

\begin{figure}[t]
  \centering
  \includegraphics[width=\textwidth]{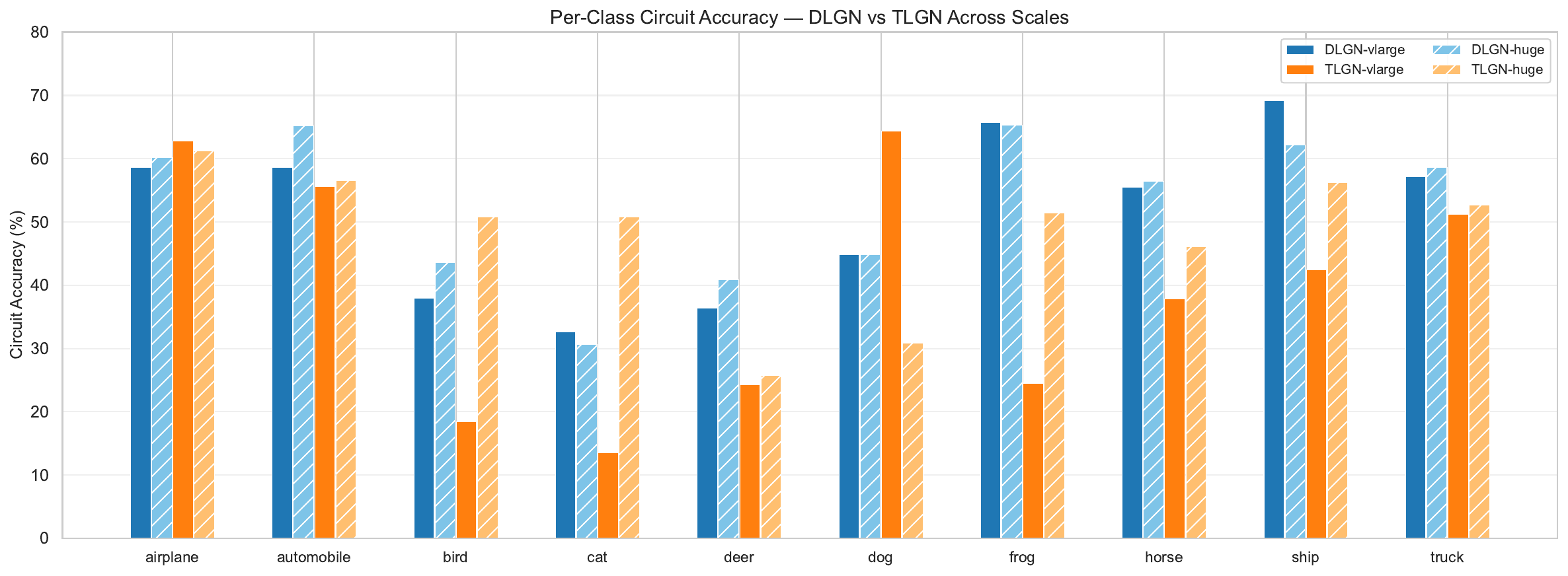}
  \caption{\textbf{Per-class circuit accuracy (\scvlarge{} and \schuge{}).}
    Blue: \dlgn{}. Orange: \tlgn{}. Solid: \scvlarge{}. Hatched: \schuge{}.
    \tlgn{}-\schuge{} recovers on visually complex classes (bird $+32.3$\,pp,
    cat $+37.2$\,pp, frog $+26.9$\,pp vs.\ \scvlarge{}), surpassing \dlgn{}
    on bird and cat.}
  \label{fig:perclass-main}
\end{figure}

\paragraph{TLGNs train 2-3$\times$ faster.}
\pst{} evaluates one polynomial per neuron (9 multiply-adds) vs.\ DLGN's softmax over 16 gates, yielding 1.5$\times$ to 3.1$\times$ speedup across scales.
PST matches binary DLGN soft accuracy while producing circuits whose hardening gap contracts with scale.

\paragraph{Takeaway.}
\pst{} is a viable and efficient training regime for ternary logic gate networks.
It matches binary \dlgn{} soft accuracy, trains 2--3$\times$ faster, and
produces circuits whose hardening gap contracts with scale. The intrinsic
advantages of ternary logic (principled abstention via $U$ outputs, a
$500\times$ richer gate vocabulary, and three-valued expressiveness) are
examined in the following subsection.

\subsection{Salient Features of Ternary LGNs}
\label{dt-lgn:subsec::results-ternary}

Binary logic gate networks (\dlgn{}) must always commit to a classification:
every neuron outputs $\{0, 1\}$, and the GroupSum aggregation always produces a
decisive class score. Ternary logic gate networks (\tlgn{}), by contrast,
operate in $\{-1, 0, +1\}$, where zero propagates through truth-table lookup
as a first-class signal. This three-valued logic endows the deployed circuit
with principled abstention, selective prediction, and Bayes-tracking
uncertainty. These capabilities have no analog in binary DLGNs and are only
accessible via \pst{}, the first training regime that makes ternary gate
networks trainable.

We validate these properties on five 2D synthetic datasets (Moons, Circles,
Spirals, Gaussians, Ring Sector; $n{=}2{,}000$ train, $500$ test) using
matched architectures: body $[512]^3$, output${}=200$,
GroupSum($k{=}2$, $\tau{=}10.0$), resolution-4 encoding ($K{=}3$, input dim${}=6$).
Both architectures are trained for 5,000 steps with discrete inputs.

\subsubsection{The $U$ signal: principled abstention absent in binary DLGNs}
\label{sec:td-contrast}



Table~\ref{tab:combined} presents the head-to-head comparison.
Binary \dlgn{} achieves higher raw accuracy on four of five datasets ($+4.7$ to $+18.8$\,pp), but this comparison is misleading: the ternary circuit deliberately abstains on ambiguous inputs via its $U$ mechanism, concentrating zero-valued outputs at decision boundaries (Figure~\ref{fig:highlight}).
Under margin-based selective prediction, the ternary circuit surpasses binary full-coverage accuracy: at 50\% coverage, Moons achieves 98.1\% (vs.\ binary's 91.8\%), Gaussians 99.5\%, and Ring Sector 100.0\%.
This demonstrates that ternary circuits provide more reliable predictions on inputs they choose to classify, a capability unavailable to binary circuits (see Appendix~\ref{sec:td-boundaries}--\ref{sec:td-coverage} for full decision boundary gallery and accuracy-vs-coverage curves).

\begin{table}[t]
  \centering
  \caption{\textbf{Ternary vs.\ binary: accuracy, abstention, and selective prediction.}
    $\Delta$: ternary $-$ binary accuracy. UNK\%: fraction of ternary output
    neurons producing zero. Acc@$k$\%: ternary accuracy when retaining
    only the $k$\% most confident samples (binary has no analogous mechanism).}
  \label{tab:combined}
  \small
  \begin{tabular}{@{}lccc|cccc@{}}
    \toprule
    & \multicolumn{3}{c|}{Raw Accuracy} & \multicolumn{4}{c}{Ternary Selective Prediction} \\
    Dataset & Bin & Tern & $\Delta$ & UNK\% & Acc@90\% & Acc@50\% & AUC \\
    \midrule
    Moons       & 91.8\% & 85.8\% & $-6.0$  & 52.9\% & 90.2\% & \textbf{98.1\%} & $-0.955$ \\
    Circles     & 97.0\% & 78.2\% & $-18.8$ & 45.1\% & 76.7\% & 84.9\% & $-0.864$ \\
    Spirals     & 63.2\% & 64.2\% & $+1.0$  & 50.1\% & 65.4\% & 78.2\% & $-0.790$ \\
    Gaussians   & 90.5\% & 78.8\% & $-11.8$ & 50.6\% & 78.8\% & \textbf{99.5\%} & $-0.846$ \\
    Ring Sector & 97.0\% & 92.2\% & $-4.7$  & 38.3\% & 96.9\% & \textbf{100.0\%}& $-0.986$ \\
    \bottomrule
  \end{tabular}
\end{table}

\begin{figure}[t]
  \centering
  \includegraphics[width=\textwidth]{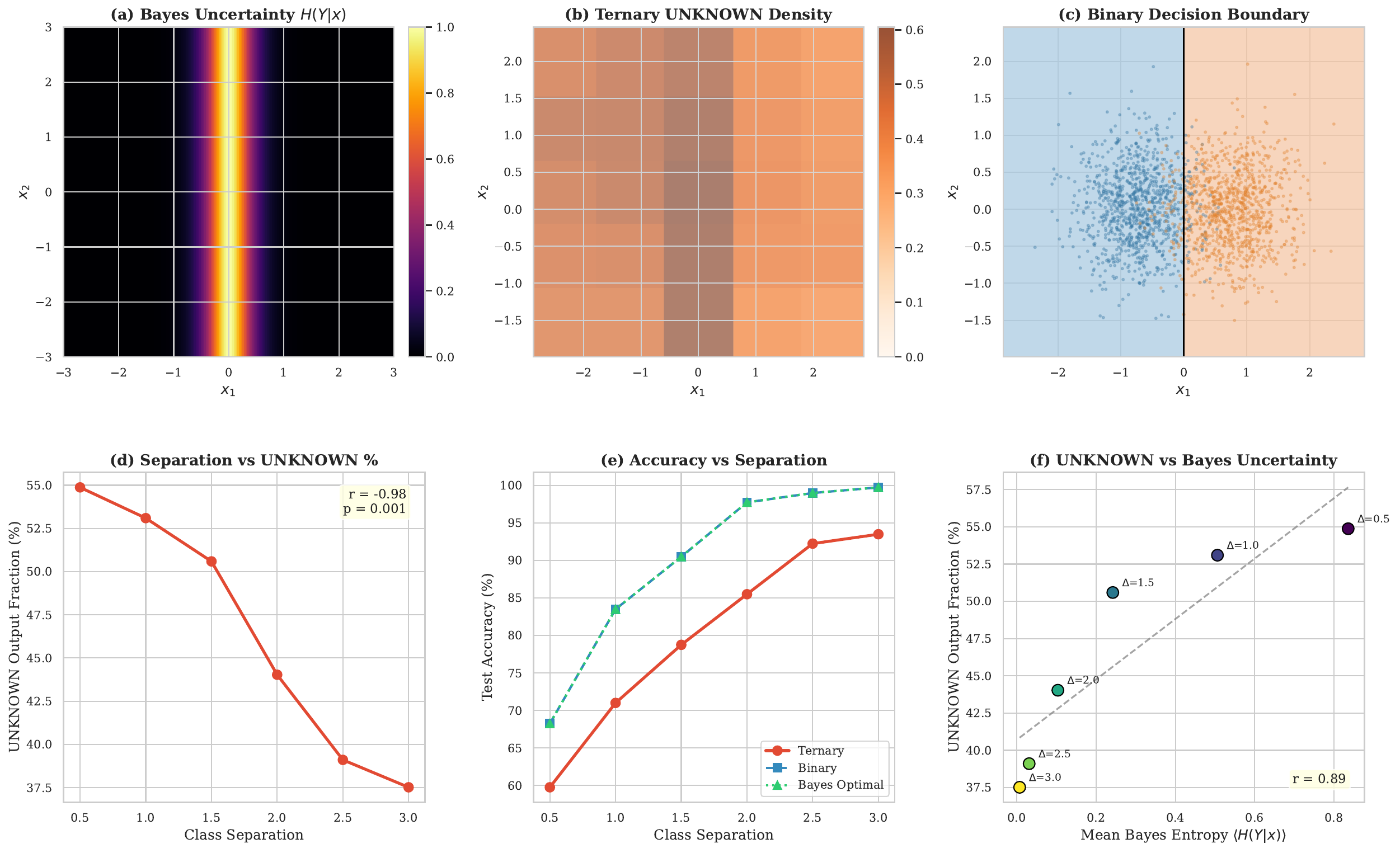}
  \caption{\textbf{Ternary $U$ tracks Bayes-optimal uncertainty.}
    Top row: Bayes posterior entropy, ternary $U$ density, and their
    overlay for Gaussians (separation${}=1.5$). Bottom row: as class
    separation increases from 0.5 to 3.0$\sigma$, $U$ fraction and
    Bayes error decrease in lockstep. Binary accuracy (red) matches the
    Bayes rate at every separation, confirming that neither architecture
    is capacity-limited; the ternary circuit \emph{chooses} to
    abstain where the posterior is ambiguous.}
  \label{fig:highlight}
\end{figure}

\paragraph{$U$ as a Bayes-optimal uncertainty proxy}
\label{sec:td-bayes}

To confirm that the $U$ signal reflects genuine statistical uncertainty
rather than training failure, we parametrically vary the difficulty of the
Gaussian classification task (Appendix Table~\ref{tab:separation-sweep}).
As class separation increases from 0.5$\sigma$ to 3.0$\sigma$, three quantities
decrease together: the Bayes error, the ternary $U$ fraction, and the
spatial overlap between $U$ density and the Bayes entropy $H(y \mid x)$
(Figure~\ref{fig:highlight}). The binary baseline matches the Bayes rate at
every separation level, confirming that these tasks are within both
architectures' capacity. The ternary circuit trades coverage for accuracy
in precisely the regions where the Bayes posterior is ambiguous, a behavior
that emerges naturally from the three-valued signal domain without any
explicit calibration objective.

\subsubsection{PST delivers ternary circuits with near-zero hardening gap}
\label{sec:td-fidelity}

The practical value of these ternary-specific features depends on whether
\pst{} can faithfully transfer the trained behavior to the deployed circuit.
Table~\ref{tab:combined} includes the hardening gap
($|\text{soft} - \text{circuit}|$) for both architectures:
the binary gap is 0.00\% on four of five datasets (3.19\% on Spirals);
the ternary gap is 0.00\% on Moons, 0.31\% on Gaussians, and 0.38\%
on Ring Sector. These near-zero gaps confirm that the polynomial surrogate
training regime, including quadratic lambda annealing and ternary commitment
regularization, produces circuits that preserve the training-time
abstention behavior after hardening.
The elevated ternary gaps on Circles (8.31\%) and Spirals (4.31\%) are
consistent with the narrow hardening basins identified in the loss landscape
analysis (Appendix Figure~\ref{fig:hardening-gap}).

  \section{Conclusion and Future Work}
\label{dt-lgn:sec::conclusion}

We introduced Polynomial Surrogate Training (PST), which makes ternary logic gate networks practical by parameterizing the full 19,683-gate vocabulary with just 9 coefficients per neuron.
PST provides provable bounds on per-neuron discretization error through a data-independent commitment loss and trains ternary circuits 2--3$\times$ faster than binary DLGNs.
Ternary logic's key advantage is principled abstention via \UNKNOWN{} outputs, which enables selective prediction where circuits surpass binary accuracy by filtering low-confidence predictions: ternary circuits achieve 98.1\% accuracy at 50\% coverage (vs.\ 91.8\% binary full-coverage accuracy on Moons), with \UNKNOWN{} density tracking Bayes-optimal uncertainty across varying task difficulty, demonstrating that three-valued logic is a practical mechanism for uncertainty-aware inference in deployed circuits.

Two directions warrant immediate attention: (a) adapting straight-through estimation and Gumbel-noise regularization~\citep{bengio2013estimating,yousefi2025mind} to PST's polynomial regime may further narrow the network-level hardening gap that contracts empirically with scale but lacks theoretical guarantees; and, (b) extending PST to recurrent architectures would enable online temporal monitoring with three-valued logic, as Signal Temporal Logic (STL) specifications inherently require ternary verdicts (\TRUE{}, \FALSE{}, \UNKNOWN{}) over finite observation windows, making recurrent DTLGNs a natural fit for runtime verification in safety-critical systems where partial observations demand principled abstention rather than forced binary classification.

  \bibliography{bib_difflogic_25}

  \clearpage
  \appendix
  \section{Hardening Gap: Lattice Geometry}
\label{appendix:lattice}

\subsection{Truth-Table Lattice Quantization}

Both the softmax-over-gates and PST training regimes face the same geometric problem: after training, each neuron must be assigned one of finitely many discrete gates.
We formalize this as a quantization problem on a truth-table lattice.

\begin{definition}[Truth-Table Lattice]
\label{def:tt-lattice-app}
Let $\mathcal{Q} = \{v_0, v_1, \ldots, v_{q-1}\} \subset \R$ be an equally spaced set of $q$ truth values.
The \emph{truth-table lattice} $\Lambda_q = \mathcal{Q}^{q^2}$ is the set of all valid truth tables for two-input $q$-logic gates, embedded in $\R^{q^2}$.
The lattice has $|\Lambda_q| = q^{q^2}$ points.
\end{definition}

For the cases of interest: Boolean logic has $q = 2$, $|\Lambda_2| = 2^4 = 16$ gates in $\R^4$; ternary logic has $q = 3$, $|\Lambda_3| = 3^9 = 19{,}683$ gates in $\R^9$; and quaternary logic has $q = 4$, $|\Lambda_4| = 4^{16} \approx 4.3 \times 10^9$ gates in $\R^{16}$.

Hardening maps each neuron's continuous output (a point in $\R^{q^2}$) to the nearest lattice point in $\Lambda_q$.
The per-neuron \emph{hardening error} is the $\ell_2$ distance from the neuron's soft truth table to the nearest valid gate.
The key geometric quantity is the \emph{covering radius} of $\Lambda_q$: the maximum distance from any point in the ambient space to the nearest lattice point.

\begin{proposition}[Covering Radius]
\label{prop:covering-radius-app}
The covering radius of the lattice $\Lambda_q$ (with truth values uniformly spaced by $\Delta = 2/(q-1)$ in $[-1,1]$) satisfies
\[
    r_{\mathrm{cov}}(\Lambda_q) = \frac{q}{q-1},
\]
and the maximum per-entry rounding error is
\[
    \epsilon_q = \frac{1}{q-1}.
\]
\end{proposition}
\begin{proof}
Each coordinate of the truth table independently rounds to the nearest value in $\mathcal{Q}$.
With $q$ uniformly spaced values in $[-1, 1]$, consecutive values are separated by $\Delta = 2/(q-1)$.
The worst-case per-coordinate error is $\Delta/2 = 1/(q-1)$.
The $\ell_2$ covering radius over all $q^2$ coordinates is $\sqrt{q^2 \cdot (\Delta/2)^2} = q \cdot \Delta/2 = q/(q-1)$.
\end{proof}

For ternary logic ($q = 3$), $\epsilon_3 = 1/2 = 0.5$; for quaternary ($q = 4$), $\epsilon_4 = 1/3 \approx 0.333$.
As $q$ increases, each coordinate need only be accurate to within $O(1/q)$ to snap to the correct gate entry.

\begin{proposition}[PST Architectural Scaling]
\label{prop:pst-scaling-app}
A $q$-logic PST neuron over two inputs requires $q^2$ learnable coefficients.
The maximum per-entry rounding error at hardening time is $\epsilon_q = 1/(q-1) = O(1/q)$.
Therefore, a higher-valence PST architecture achieves tighter rounding tolerance per truth-table entry at a cost that grows only quadratically in $q$.
By contrast, a softmax-over-gates neuron for $q$-ary logic requires $K = q^{q^2}$ logits (super-exponential in $q$), and its rounding error is governed by the concentration of a categorical distribution over $K$ categories.
\end{proposition}

PST's design improves with logic valence: the polynomial representation cost grows polynomially ($q^2$ coefficients) while the lattice becomes denser ($\epsilon_q \to 0$).
The softmax-over-gates regime faces the opposite situation: its parameterization cost explodes super-exponentially while the concentration problem worsens.

\section{Spectral Analysis of Ternary Logic Gate Networks}
\label{appendix:spectral}

\noindent \textbf{Computing the Orthonormal Basis.} We have the ternary set $\mathcal{T}=\{-1,0,+1\}$ and want to represent every function $f: \mathcal{T}^2 \rightarrow \mathbb{R}$ as a linear combination of basis functions. Since $\left|\mathcal{T}^2\right|=9$, the space of such functions is $\mathbb{R}^9-$ nine-dimensional. We need exactly 9 basis functions, and we want them orthogonal under some inner product.

\noindent \textbf{The Inner Product.} The uniform measure on $\mathcal{T}$ assigns probability $1 / 3$ to each value. The inner product on functions $f, g: \mathcal{T} \rightarrow \mathbb{R}$ is:

\begin{equation*}
\langle f, g\rangle=\frac{1}{3} \sum_{x \in \mathcal{T}} f(x) g(x)=\frac{1}{3}[f(-1) g(-1)+f(0) g(0)+f(1) g(1)]
\end{equation*}

\noindent Each ternary input is equally weighted. The factor $1 / 3$ is a normalization - it makes $\langle 1,1\rangle =1$ so the constant function has unit norm. The bivariate inner product on $\mathcal{T}^2$ uses the same idea with factor 1/9.

\noindent \textbf{The Monomial Basis.} The monomials $\left\{1, x, x^2\right\}$ are a valid basis for polynomials of degree $\leq 2$ in one variable, and they span the right space (any function $\mathcal{T} \rightarrow \mathbb{R}$ is determined by 3 values, and 3 monomials give 3 degrees of freedom). But they are not orthogonal under this inner product:

\begin{equation*}
\left\langle 1, x^2\right\rangle=\frac{1}{3}[1 \cdot 1+1 \cdot 0+1 \cdot 1]=\frac{2}{3} \neq 0
\end{equation*}

\noindent The constant function and $x^2$ are correlated on $\mathcal{T}$ because $x^2$ evaluates to 1 at both $\pm 1$ and 0 at 0 . Under the uniform measure, $x^2$ has mean $2 / 3$ - it is biased toward 1 because two out of three ternary inputs are extreme ( $\pm 1$ ) and only one is neutral ( 0 ). Since $x^2$ is nonzero at two-thirds of the domain, its mean is nonzero, which means it is not orthogonal to the constant.
The monomial basis is fine for computation (and is what we use for the weight parameterization), but it is unsuitable for analysis because the coefficients are entangled and we cannot ascribe meaning to the values of these coefficients like we can with the Fourier coefficients.

\noindent \textbf{Gram-Schmidt Orthogonalization.} We generate the orthonormal basis from the monomial basis by following the Gram-Schmidt process:

\noindent Step 1: $\varphi_0(x)=1$. Norm: $\left\|\varphi_0\right\|^2=\langle 1,1\rangle=1$.

\noindent Step 2: Start with x , subtract its projection onto $\varphi_0$.

\begin{equation*}
\left\langle x, \varphi_0\right\rangle=\frac{1}{3}(-1 \cdot 1+0 \cdot 1+1 \cdot 1)=0
\end{equation*}

\noindent The projection is zero; x is already orthogonal to the constant. This happens because $x$ is an odd function and the measure is symmetric around 0. So $\varphi_1(x)=x$, with $\left\|\varphi_1\right\|^2==2 / 3$.

\noindent Step 3: Start with $x^2$, subtract projections onto $\varphi_0$ and $\varphi_1$.

\begin{equation*}
\left\langle x^2, \varphi_1\right\rangle=\frac{1}{3}\left((-1)^2(-1)+0^2 \cdot 0+1^2 \cdot 1\right)=0
\end{equation*}

\noindent Zero by symmetry $-x^2$ is even, $x$ is odd, their product is odd, and odd functions sum to zero on $\mathcal{T}$. So we only need to subtract the $\varphi_0$ projection:

\begin{equation*}
\begin{gathered}
\left\langle x^2, \varphi_0\right\rangle=\frac{1}{3}(1+0+1)=\frac{2}{3} \\
\varphi_2(x)=x^2-\frac{\left\langle x^2, \varphi_0\right\rangle}{\left\|\varphi_0\right\|^2} \cdot \varphi_0=x^2-\frac{2 / 3}{1} \cdot 1=x^2-\frac{2}{3}
\end{gathered}
\end{equation*}

\noindent The $-2 / 3$ is not arbitrary - it is the mean of $x^2$ under the uniform measure on $\mathcal{T}$. The polynomial $\varphi_2(x)=x^2-2 / 3$ is the ``centered'' quadratic: it measures whether $x$ is extreme $\left( \pm 1\right.$, where $\left.\varphi_2=1 / 3\right)$ versus neutral $\left(0\right.$, where $\left.\varphi_2=-2 / 3\right)$. It has zero mean on $\mathcal{T}$ by construction with $\left\|\varphi_2\right\|^2=2 / 9$. In Boolean analysis on $\{-1,+1\}$, every function of one variable is a linear combination of 1 and x - there are only two basis functions because $|\{-1,+1\}|=2$. The quadratic $\varphi_2$ is the "extra dimension" that ternary logic provides, and it captures exactly the UNKNOWN-sensitivity that makes Kleene K3 logic different from Boolean logic.

 \begin{remark}
     Our basis differs from the additive-combinatorial convention in two ways: (i) we use \emph{real-valued} orthogonal polynomials rather than complex characters, matching the real polynomial parameterization of PST neurons; and (ii) we construct the basis via Gram-Schmidt from $\{1, x, x^2\}$ rather than from the group characters of $\mathbb{Z}_3$, yielding a basis that is adapted to the Kleene interpretation where $\{-1, +1\}$ are ``decided'' values and $0$ is ``undetermined.''
    The resulting quadratic basis function $\varphi_2 = x^2 - 2/3$ is orthogonal to both the constant and linear terms and captures precisely the sensitivity to \UNKNOWN{} that has no analogue in Boolean Fourier analysis.
 \end{remark}

\noindent The full list of the Fourier coefficients is as follows: $\Phi_{00} =1, \Phi_{10} =x, \Phi_{01} =y, \Phi_{11} =xy, \Phi_{20} =x^2-2/3, \Phi_{02} =y^2-2/3, \Phi_{21} =(x^2-2/3)y, \Phi_{12} =x(y^2-2/3), \Phi_{22} =(x^2-2/3)(y^2-2/3)$.\\

\paragraph{Spectral complexity classes.}
The Fourier spectrum naturally stratifies the 19,683 ternary gates into complexity classes.
\textbf{Linear gates} have energy only in $\hat{f}_{00}, \hat{f}_{10}, \hat{f}_{01}$. These include MAJORITY and the trivial pass-through gates, and they are the only gates that can be computed by a single linear threshold on ternary inputs.
\textbf{Bilinear gates} additionally use the interaction term $\hat{f}_{11}$; AND, OR, XOR, and their variants fall here.
\textbf{Quadratic gates} involve $\hat{f}_{20}$ or $\hat{f}_{02}$ and are the genuinely ``ternary'' gates that distinguish between decided and undecided inputs.
\textbf{Full gates} use all nine terms.
The Fourier $L_1$ norm $\|\hat{f}\|_1 = \sum_{ij} |\hat{f}_{ij}|$ is a rigorous complexity measure: it upper-bounds the number of random examples needed to learn the gate to constant accuracy, and it determines the gate's sensitivity to random input perturbations.

\paragraph{Fourier sparsity regularization.}
The $L_1$ norm of Fourier coefficients serves as a sparsity regularizer independent of the commitment loss:
\begin{equation}
\label{eq:fourier-reg}
\mathcal{R}_F(\mathbf{W}) = \frac{1}{N}\sum_{j=1}^{N} \sum_{i,k} |\hat{f}_{ik}^{(j)}|.
\end{equation}
The commitment loss $\mathcal{R}_A$ drives neurons toward \emph{any} of the 19,683 gates; the Fourier regularizer biases that choice toward spectrally sparse, interpretable gates.
The combined training objective is $\loss = \loss_{\mathrm{task}} + \lambda(t) \cdot \mathcal{R}_A + \beta \cdot \mathcal{R}_F$.

\paragraph{Post-hoc spectral profiling.}
After training and discretization, the aggregate spectral energy distribution across all neurons reveals whether the task genuinely exploits ternary structure.
If most energy concentrates in linear Fourier terms ($\hat{f}_{10}$, $\hat{f}_{01}$, $\hat{f}_{11}$), the network is essentially performing sign propagation and could function as a Boolean circuit.
If substantial energy appears in quadratic terms ($\hat{f}_{20}$, $\hat{f}_{02}$, $\hat{f}_{22}$), the network is genuinely exploiting the three-valued distinction, routing information differently based on whether inputs are decided versus undetermined.
This spectral profile is a principled diagnostic for whether $K_3$ logic adds value over binary logic for a given task.

\paragraph{Vocabulary coverage.}
The Hamming distance between a neuron's discretized truth table and the nearest curated gate quantifies vocabulary coverage.
The Level~2 coverage (the fraction of neurons whose nearest curated gate has Hamming distance zero) reveals which tasks are ``natively ternary'' (high coverage) versus tasks requiring gates outside the standard vocabulary. One of the more interesting, formally defined vocabularies of ternary gates is that of \emph{threshold logic} \citep{jones2012ternary}. In threshold logic, each input $i_j$ is multiplied by a constant $k_j$ and then the inputs are summed and then compared to two thresholds $t_+$ and $t_–$. If the sum is $t_+$ or greater, the output is $+1$. If the sum is $t_–$, the output is $–1$. According to \cite{merrill1965ternary}, 471 of the 19,683 diadic ternary operators can be computed using threshold logic.

\section{Additional Experimental Results}
\label{appendix:expts_results}
\subsection{Supplementary: CIFAR-10 scaling details}
\label{sec:cifar-appendix}

\begin{table}[t]
  \centering
  \caption{\textbf{Architecture configurations.} All use GroupSum($k{=}10$,
    $\tau{=}33.3$) for 10-class output. Input dimension: 9,216.}
  \label{tab:configs}
  \small
  \begin{tabular}{@{}llrrr@{}}
    \toprule
    Scale & Widths & Neurons & \pst{} LR & Steps \\
    \midrule
    \scsmall{}   & $[12000,\;12000,\;12000,\;12000]$ &  48,000 & 0.003 & 100K \\
    \scmedium{}  & $[24000,\;24000,\;24000,\;24000]$ &  96,000 & 0.001 & 150K \\
    \sclarge{}   & $[36000,\;36000,\;36000,\;36000]$ & 144,000 & 0.001 & 200K \\
    \scdeeper{}  & $[12000,\;12000,\;12000,\;12000,\;12000]$ &  60,000 & 0.003 & 100K \\
    \scvlarge{}  & $[48000,\;48000,\;48000,\;48000]$ & 192,000 & 0.001 & 150K \\
    \schuge{}    & $[128000,\;128000,\;128000,\;128000]$& 512,000 & 0.001 & 200K \\
    \bottomrule
  \end{tabular}
\end{table}

\subsubsection{Training dynamics}
\label{sec:training}

\textbf{Base scales (48K-144K).}
Figure~\ref{fig:nb07-loss} shows training loss curves for four scales.
At \scsmall{} and \scmedium{} scales, \dlgn{} enters a noisy plateau
(oscillating in 0.5-1.4), while \pst{} achieves $10$-$38\times$ lower task
loss (e.g., \tlgn{}-\sclarge{}: 0.017 vs.\ \dlgn{}-\sclarge{}: 0.65). The
5-layer \scdeeper{} variant shows the same plateau as \scsmall{}, confirming
that depth alone does not resolve it.

\begin{figure}[t]
  \centering
  \includegraphics[width=\textwidth]{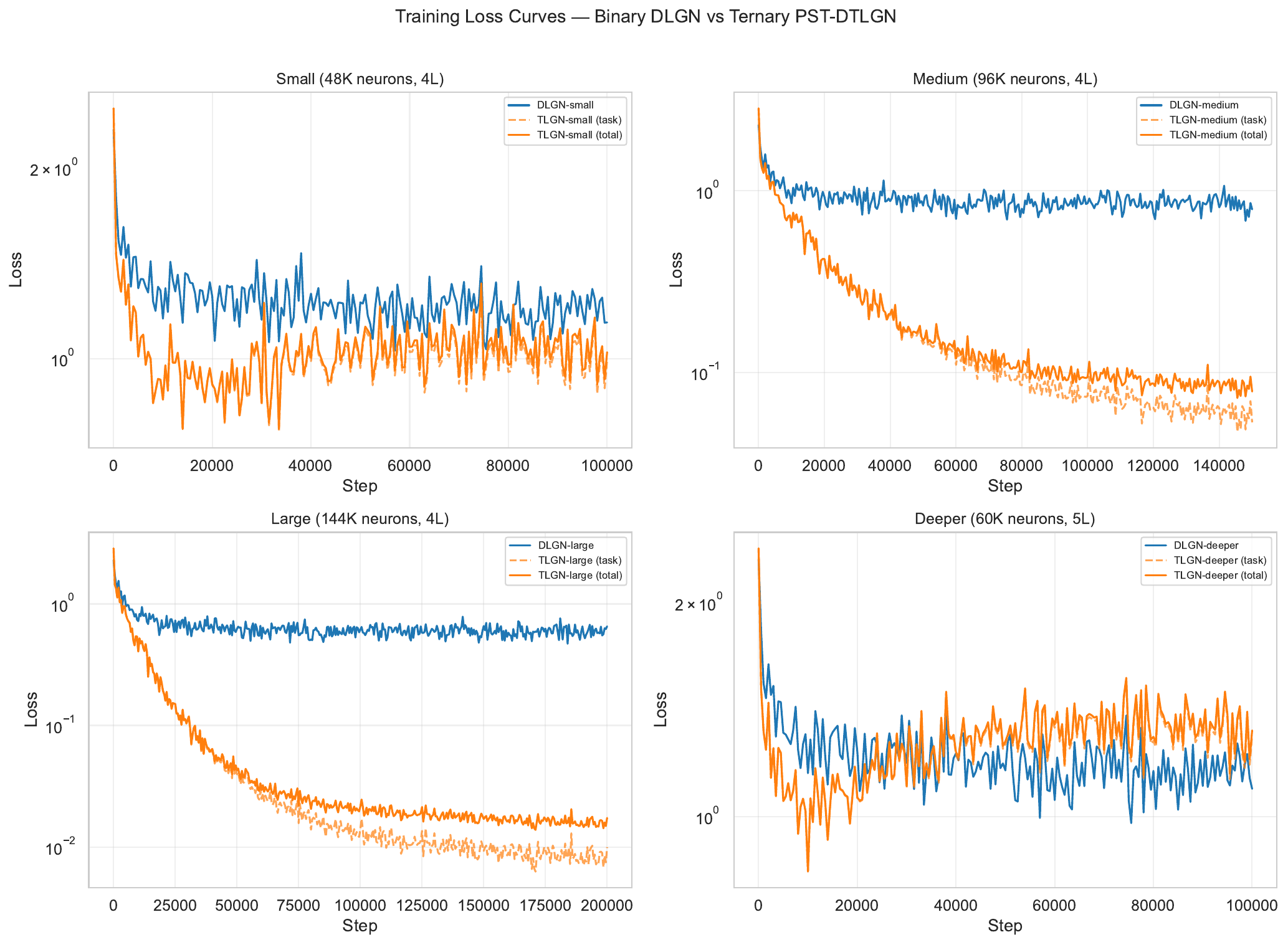}
  \caption{\textbf{Training loss curves (\scsmall{}--\sclarge{} + \scdeeper{}).}
    $2 \times 2$ grid, log-scale $y$-axis. Blue: \dlgn{}. Dashed orange:
    \tlgn{} task loss. Solid orange: \tlgn{} total loss (task + $\lambda \cdot$ commitment).
    \tlgn{} converges 1-2 orders of magnitude below \dlgn{} at \scmedium{} and
    \sclarge{} scales, while both remain comparable at \scsmall{} scale.}
  \label{fig:nb07-loss}
\end{figure}

\textbf{Extended scales (192K-512K).}
Figure~\ref{fig:nb07a-loss} extends to \scvlarge{} and \schuge{}.
\dlgn{}-\schuge{} (512K) breaks through to 0.015, a qualitative transition
from plateau-dominated to clean convergence, suggesting the softmax-over-gates
parameterization requires a threshold overparameterization level (between
48K and 128K per-layer width). \tlgn{}-\schuge{} reaches 0.004, the lowest task
loss of any model.

\begin{figure}[t]
  \centering
  \includegraphics[width=\textwidth]{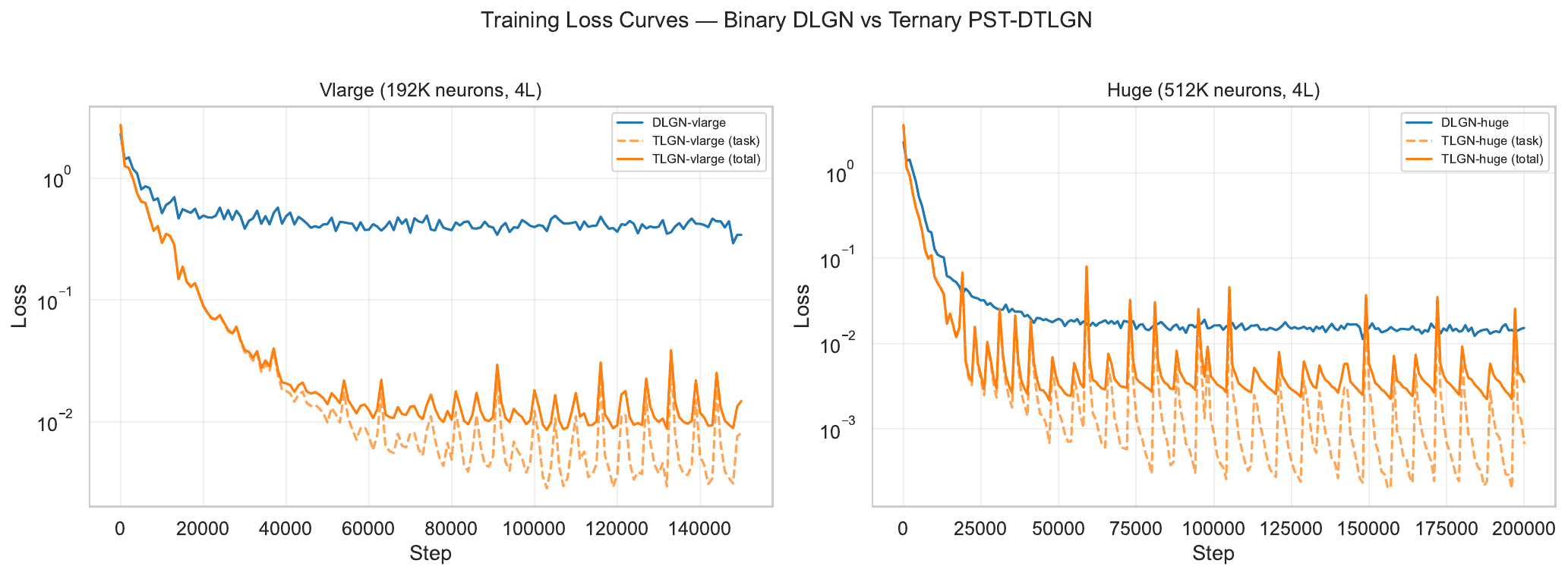}
  \caption{\textbf{Training loss curves (\scvlarge{} and \schuge{}).}
    Left: \scvlarge{} (192K neurons). Right: \schuge{} (512K neurons).
    At \schuge{} scale, \dlgn{} finally achieves clean convergence (0.015),
    eliminating the noisy-plateau pathology. \tlgn{}-\schuge{} reaches 0.004
    (task), the lowest of any model.}
  \label{fig:nb07a-loss}
\end{figure}

\textbf{Ternary convergence.}
Figure~\ref{fig:ternary-overlay} overlays \tlgn{} task loss curves for
\scvlarge{} and \schuge{}. \tlgn{}-\schuge{} reaches its floor
(${\sim}10^{-3}$) by step 30K, matching \tlgn{}-\scvlarge{}'s asymptotic
performance in 15\% of the training budget. Periodic spikes coincide with lambda
annealing ramp-up; these are transient (recovery within 1-2K steps) and do not
affect final performance.

\begin{figure}[t]
  \centering
  \includegraphics[width=0.85\textwidth]{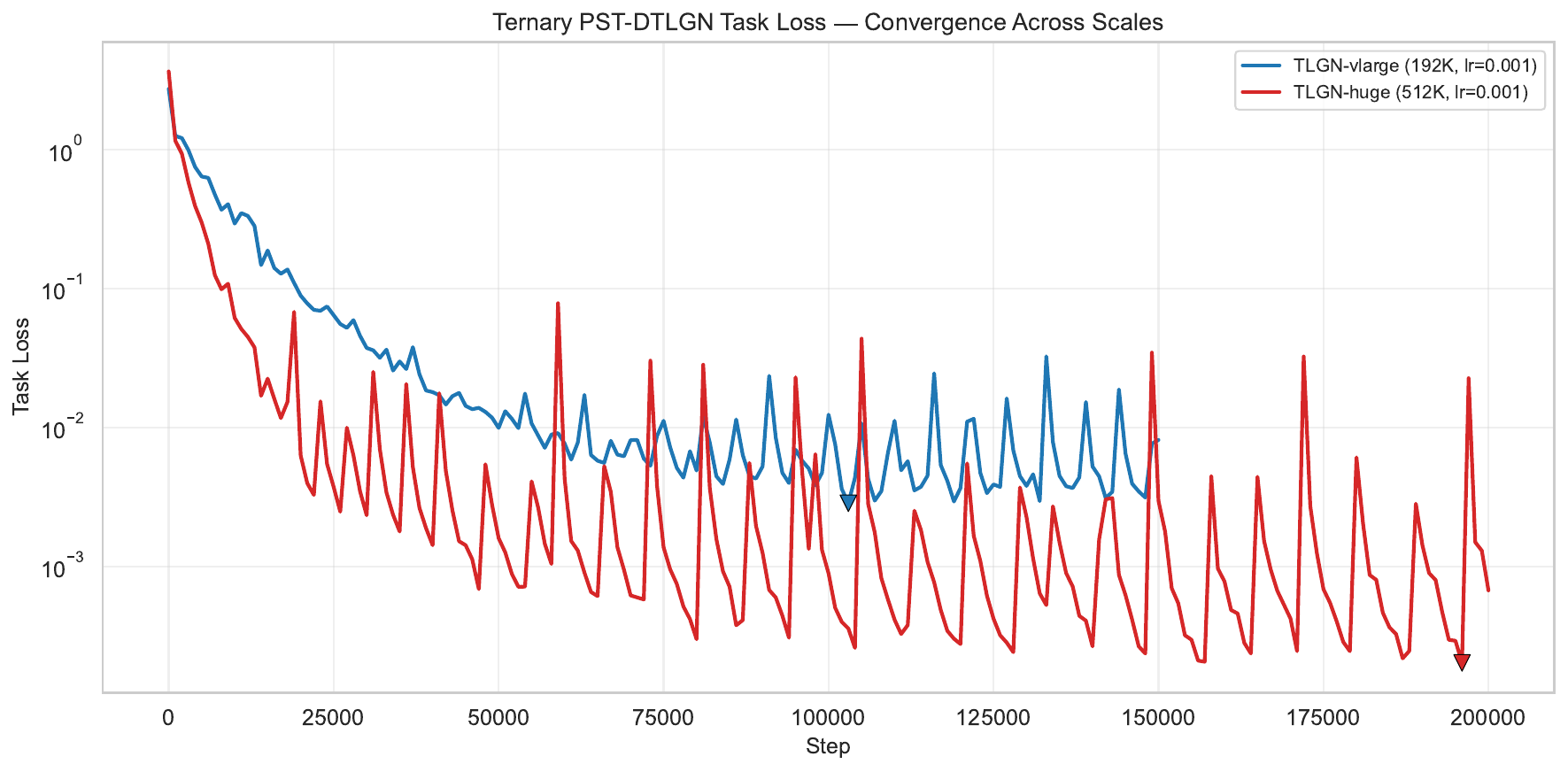}
  \caption{\textbf{Ternary task loss convergence} (\scvlarge{} and \schuge{}).
    Triangles mark the minimum. Both use lr${=}0.001$.
    \tlgn{}-\schuge{} reaches its floor (${\sim}10^{-3}$) by step 30K,
    matching \tlgn{}-\scvlarge{}'s asymptotic performance in 15\% of the
    training budget, then continues improving to $3 \times 10^{-4}$.
    Periodic spikes (visible as upward excursions) coincide with
    lambda annealing ramp-up.}
  \label{fig:ternary-overlay}
\end{figure}

\begin{table}[t]
  \centering
  \caption{\textbf{Training wall time} (single RTX 4090). Speed ratio =
    \dlgn{} time / \tlgn{} time.}
  \label{tab:speed}
  \small
  \begin{tabular}{@{}lrrrr@{}}
    \toprule
    Scale & Neurons & \dlgn{} (s) & \tlgn{} (s) & Speedup \\
    \midrule
    \scsmall{}  &  48K &  241 &  156 & 1.5$\times$ \\
    \scmedium{} &  96K &  587 &  292 & 2.0$\times$ \\
    \sclarge{}  & 144K & 1023 &  479 & 2.1$\times$ \\
    \scdeeper{} &  60K &  297 &  167 & 1.8$\times$ \\
    \scvlarge{} & 192K &  936 &  455 & 2.1$\times$ \\
    \schuge{}   & 512K & 4408 & 1401 & 3.1$\times$ \\
    \bottomrule
  \end{tabular}
\end{table}

\subsubsection{Hardening gap mechanism}
\label{sec:gap-mechanism}

Three regimes govern the gap trajectory.

\emph{Regime~I: Low capacity} (48K-60K).
Gap is modest ($2.0$-$2.8$\,pp). The polynomials have limited freedom; the gap
is small because there is little to lose in hardening. However, hardening error
is high (0.298): the polynomials are not close to valid gates, but per-neuron
rounding errors do not compound catastrophically at this scale.

\emph{Regime~II: Mid capacity} (96K-192K).
Gap peaks at $14.1$\,pp (\scmedium{}) and remains elevated through \scvlarge{}
($11.9$\,pp). The network develops rich polynomial representations far from
valid ternary truth tables. The commitment regularizer reduces hardening error
monotonically (0.253 $\to$ 0.067), but accumulated rounding through four layers
remains destructive.

\emph{Regime~III: High capacity} (512K).
Gap contracts to $3.7$\,pp with hardening error 0.029, meaning each truth table entry
is within 0.029 of the nearest ternary value on average. The contraction is
accompanied by a sharp increase in UNKNOWN output neurons ($24.6\%$ vs.\
${\sim}6\%$ at smaller scales), reflecting implicit pruning: neurons near
Voronoi cell boundaries in truth-table space are mapped to zero-output gates,
while surviving neurons carry sufficient discriminative signal.

\begin{table}[t]
  \centering
  \caption{\textbf{Hardening error and UNKNOWN\% across scales.}
    Hardening error drops $10\times$ while UNKNOWN\% increases $4\times$
    at \schuge{} scale.}
  \label{tab:gap-mechanism}
  \small
  \begin{tabular}{@{}lrrr@{}}
    \toprule
    Scale & \tlgn{} Gap (pp) & Hard Err & UNK\% \\
    \midrule
    \scsmall{}  & $+2.8$  & 0.298 &  5.9\% \\
    \scdeeper{} & $+2.0$  & 0.218 &  7.3\% \\
    \scmedium{} & $+14.1$ & 0.253 &  6.2\% \\
    \sclarge{}  & $+13.0$ & 0.074 &  6.0\% \\
    \scvlarge{} & $+11.9$ & 0.067 &  9.3\% \\
    \schuge{}   & $+3.7$  & 0.029 & 24.6\% \\
    \bottomrule
  \end{tabular}
\end{table}

\subsubsection{Gate diversity and Fourier spectral analysis}
\label{sec:appendix-diversity}

Table~\ref{tab:redundancy} reports gate diversity and functional redundancy.
The effective diversity ratio (\tlgn{}/\dlgn{}) ranges from 286$\times$
(\scsmall{}) to 523$\times$ (\sclarge{}).
Unique gates grow from 8,987 (\scsmall{}) to ${\sim}14{,}000$
(\sclarge{}-\scvlarge{}) then decrease slightly to 13,955 at \schuge{}, while
Gini increases monotonically (0.603 $\to$ 0.696) and effective diversity drops
from 7,590 to 5,437. The \schuge{} model uses fewer distinct gate types but
copies each more frequently.

\begin{table}[t]
  \centering
  \caption{\textbf{Gate diversity and functional redundancy.}
    Unique: distinct truth tables. Eff.\ Div: exponential Shannon entropy.
    Gini: concentration (higher = more concentrated).
    Max copies: neuron count of the most common gate.
    Singletons: gates used by exactly one neuron.}
  \label{tab:redundancy}
  \small
  \begin{tabular}{@{}lrrrrrr@{}}
    \toprule
    Model & Unique & Eff.\ Div & Gini & Redund.\% & Max Copies & Singletons \\
    \midrule
    \dlgn{}-\scsmall{}    &    16 &   14.9 & 0.176 & 100.0\% &  4,128 & 0 \\
    \tlgn{}-\scsmall{}    & 8,987 & 4,268  & 0.603 &  81.3\% &    224 & 3,149 \\
    \addlinespace
    \dlgn{}-\scmedium{}   &    16 &   14.7 & 0.209 & 100.0\% & 10,341 & 0 \\
    \tlgn{}-\scmedium{}   & 12,655& 6,719  & 0.587 &  86.8\% &    144 & 3,032 \\
    \addlinespace
    \dlgn{}-\sclarge{}    &    16 &   14.5 & 0.222 & 100.0\% & 16,620 & 0 \\
    \tlgn{}-\sclarge{}    & 14,375& 7,590  & 0.592 &  90.0\% &    190 & 2,578 \\
    \addlinespace
    \dlgn{}-\scvlarge{}   &    16 &   14.4 & 0.231 & 100.0\% & 23,263 & 0 \\
    \tlgn{}-\scvlarge{}   & 14,360& 7,239  & 0.610 &  92.5\% &    338 & 2,181 \\
    \addlinespace
    \dlgn{}-\schuge{}     &    16 &   14.2 & 0.251 & 100.0\% & 66,609 & 0 \\
    \tlgn{}-\schuge{}     & 13,955& 5,437  & 0.696 &  97.3\% &  1,490 & 1,710 \\
    \bottomrule
  \end{tabular}
\end{table}

Table~\ref{tab:fourier} decomposes the unique gates into Fourier complexity
bands. Larger networks learn more complex gates: linear energy decreases from
36.8\% (\scsmall{}) to 32.6\% (\sclarge{}) while quadratic and cubic energy
increase. The \scdeeper{} variant has the most linear profile (37.7\%),
consistent with narrower per-layer width limiting inter-neuron interaction
complexity. Over 98\% of gates are genuinely ternary at all scales.

\begin{table}[t]
  \centering
  \caption{\textbf{Fourier spectral profile} (\tlgn{} models).
    Percentage of total Fourier energy in each degree band.
    \%Ternary: fraction of unique gates that are genuinely ternary
    (not binary-equivalent).}
  \label{tab:fourier}
  \small
  \begin{tabular}{@{}lrcrrrrr@{}}
    \toprule
    Model & Unique & \%Ternary & Const & Linear & Quad & Cubic & Quartic \\
    \midrule
    \tlgn{}-\scsmall{}  & 8,987  & 98.6\% & 18.1\% & 36.8\% & 31.4\% & 11.0\% & 2.8\% \\
    \tlgn{}-\scmedium{} & 12,655 & 99.3\% & 17.4\% & 33.5\% & 33.1\% & 12.6\% & 3.3\% \\
    \tlgn{}-\sclarge{}  & 14,375 & 99.4\% & 17.0\% & 32.6\% & 33.5\% & 13.4\% & 3.5\% \\
    \tlgn{}-\scdeeper{} & 9,068  & 96.8\% & 23.6\% & 37.7\% & 26.8\% &  9.4\% & 2.5\% \\
    \bottomrule
  \end{tabular}
\end{table}

\subsubsection{Gate landscape geometry via MDS embedding}
\label{sec:mds}

The Fourier spectral profile (Table~\ref{tab:fourier}) summarizes
\emph{aggregate} complexity but does not reveal how the ${\sim}9{,}000$--$14{,}000$
unique gates are organized relative to one another. We apply multidimensional
scaling (MDS) under two complementary distance metrics: Hamming distance on
truth tables; and, Euclidean ($L^{2}$) distance on Fourier coefficient vectors;
to embed each unique gate as a point in $\mathbb{R}^{2}$.
Figure~\ref{fig:mds} shows the result for all four base scales.

\begin{figure}[t]
  \centering
  \includegraphics[width=\textwidth]{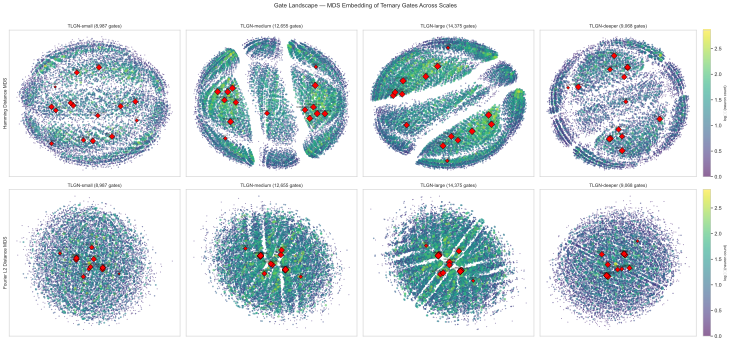}
  \caption{\textbf{Gate landscape - MDS embedding of ternary gates across
    scales.}
    Top row: Hamming distance (number of differing truth table entries out of~9).
    Bottom row: Fourier $L^{2}$ distance.
    Each point is a unique gate; area $\propto \log_{2}(\text{neuron count})$;
    colour encodes $\log_{10}(\text{neuron count})$.
    Red diamonds: the 15 named curated gates (AND, OR, XOR, etc.).
    Curated gates occupy the high-count core in all four embeddings,
    confirming that \pst{} rediscovers classical logic operations without
    any gate-vocabulary constraint.}
  \label{fig:mds}
\end{figure}

\paragraph{Clustering of learned gates around named gates.}
Under both metrics and at all four scales, the highest-count learned gates
(bright points, $\log_{10}(\text{count}) > 2$) cluster in the same region
as the 15 named curated gates (red diamonds). This co-location indicates
that \pst{}, despite searching an unconstrained space of 19,683 truth
tables, independently converges toward the neighbourhood of classical logic
operations (AND, OR, XOR, NAND, etc.). Low-count gates --- including the
${\sim}2{,}600$ singletons at \sclarge{} (Table~\ref{tab:redundancy}) -
are scattered at greater embedding distances from this core, consistent
with their role as rare, task-specific functions.

The two metrics provide complementary views. Hamming distance captures
how many truth table entries differ between gates (out of~9), so nearby
points in the top row are \emph{functionally} similar. Fourier $L^{2}$
distance captures differences in spectral energy distribution, so nearby
points in the bottom row share similar complexity profiles regardless of
which specific entries differ. The named gates occupy the high-count
core under both metrics, confirming that they are central in both the
functional and spectral sense.

\paragraph{Scale dependence.}
From \scsmall{} to \sclarge{}, the gate cloud grows denser as unique gates
increase (8,987 $\to$ 14,375) and effective diversity rises
(4,268 $\to$ 7,590; Table~\ref{tab:redundancy}).
The \scdeeper{} variant (rightmost column), despite a comparable neuron
count to \scsmall{}, produces a more compact cloud, consistent with its
higher constant-energy fraction (23.6\% vs.\ 18.1\%) and lower
quadratic-energy fraction (26.8\% vs.\ 31.4\%; Table~\ref{tab:fourier}).

\subsubsection{Per-class analysis}
\label{sec:perclass}

The per-class breakdown (Figure~\ref{fig:perclass}, Table~\ref{tab:perclass})
reveals that the \tlgn{} hardening gap is not uniform across classes.

\begin{figure}[t]
  \centering
  \includegraphics[width=\textwidth]{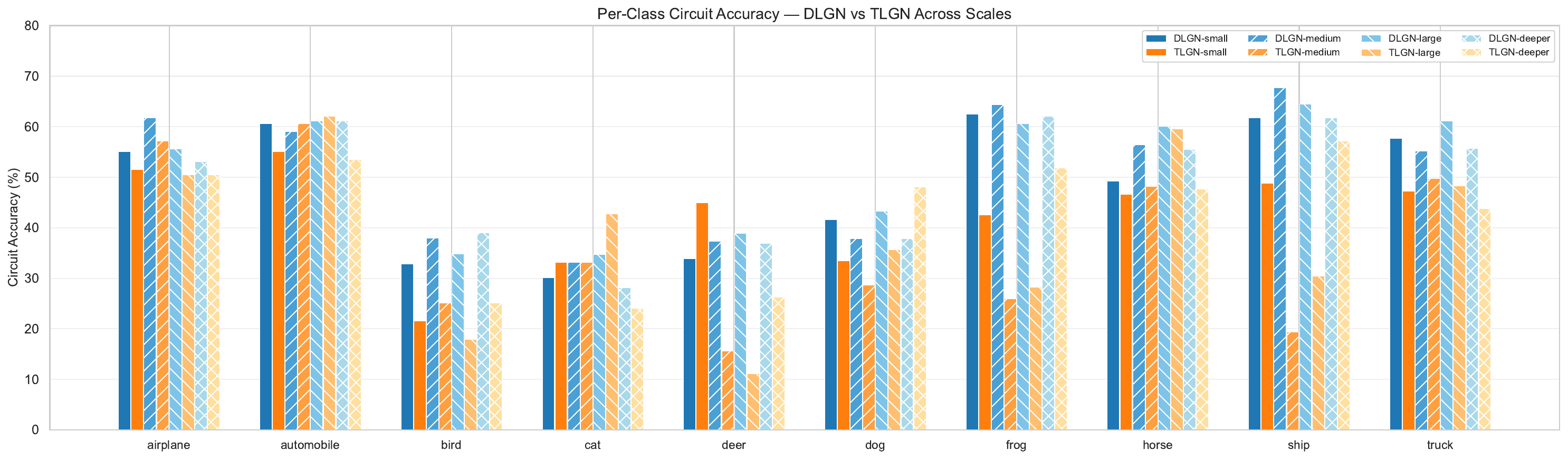}
  \caption{\textbf{Per-class circuit accuracy at smaller scales.} Blue: \dlgn{}. Orange: \tlgn{}.
    Solid: lower scale. Hatched: higher scale.
    \tlgn{} shows high class-level variance at \scsmall{}-\scdeeper{} scales,
    with catastrophic failures on visually complex classes (bird, cat, deer, frog).
    For larger scales (\scvlarge{}, \schuge{}) showing partial recovery, see Figure~\ref{fig:perclass-main} in the main text.}
  \label{fig:perclass}
\end{figure}

\begin{table}[t]
  \centering
  \caption{\textbf{Per-class circuit accuracy (\%) at selected scales.}
    Bold: best ternary result per class. $\Delta$: improvement from \scvlarge{}
    to \schuge{}. Classes with \scvlarge{} accuracy $< 30\%$ are marked
    with $\dagger$.}
  \label{tab:perclass}
  \small
  \begin{tabular}{@{}lcccc|ccr@{}}
    \toprule
    & \multicolumn{4}{c|}{\tlgn{}} & \multicolumn{2}{c}{\dlgn{} ref.} & \\
    Class & \scriptsize \scsmall{} & \scriptsize \sclarge{} &
            \scriptsize \scvlarge{} & \scriptsize \schuge{} &
            \scriptsize \scvlarge{} & \scriptsize \schuge{} & $\Delta_{\text{vl}\to\text{h}}$ \\
    \midrule
    airplane    & 51.5 & 50.5 & 62.8 & \textbf{61.2} & 58.7 & 60.2 & $-1.6$ \\
    automobile  & 55.1 & 62.1 & 55.6 & \textbf{56.6} & 58.6 & 65.2 & $+1.0$ \\
    bird$^\dagger$  & 21.5 & 17.9 & 18.5 & \textbf{50.8} & 37.9 & 43.6 & $+32.3$ \\
    cat$^\dagger$   & 33.2 & 42.7 & 13.6 & \textbf{50.8} & 32.7 & 30.7 & $+37.2$ \\
    deer$^\dagger$  & 44.9 & 11.1 & 24.2 & \textbf{25.8} & 36.4 & 40.9 & $+1.6$ \\
    dog         & 33.5 & 35.7 & \textbf{64.3} & 30.8 & 44.9 & 44.9 & $-33.5$ \\
    frog$^\dagger$  & 42.6 & 28.2 & 24.5 & \textbf{51.4} & 65.7 & 65.3 & $+26.9$ \\
    horse       & 46.6 & 59.6 & 37.8 & \textbf{46.1} & 55.4 & 56.5 & $+8.3$ \\
    ship        & 48.8 & 30.4 & 42.4 & \textbf{56.2} & 69.1 & 62.2 & $+13.8$ \\
    truck       & 47.3 & 48.3 & 51.2 & \textbf{52.7} & 57.1 & 58.6 & $+1.5$ \\
    \bottomrule
  \end{tabular}
\end{table}

\textbf{\tlgn{}-\schuge{} recovers on hard classes.}
The most dramatic improvements from \scvlarge{} to \schuge{} occur on visually
complex classes: bird ($+32.3$\,pp), cat ($+37.2$\,pp), frog ($+26.9$\,pp),
and ship ($+13.8$\,pp). At \schuge{} scale, \tlgn{} surpasses \dlgn{} on bird
(50.8\% vs.\ 43.6\%) and cat (50.8\% vs.\ 30.7\%).

\textbf{Dog class regression.}
\tlgn{}-\scvlarge{} achieved 64.3\% on dog but \tlgn{}-\schuge{} drops to
30.8\% ($-33.5$\,pp). The dog-cat confusability in CIFAR-10 and the $4\times$
increase in UNKNOWN neurons suggest a redistribution of discriminative capacity
rather than uniform improvement.

\textbf{\dlgn{} is more uniform.}
Binary circuit accuracy has lower per-class variance (CoV: 0.23 for
\dlgn{}-\schuge{} vs.\ 0.28 for \tlgn{}-\schuge{}) and dominates on ``easy''
classes (frog: 65.3\% vs.\ 51.4\%; automobile: 65.2\% vs.\ 56.6\%).

\subsubsection{Margin-based confidence}
\label{sec:margins}

Ternary circuits provide a natural confidence signal via the GroupSum margin
(difference between highest and second-highest class scores).

\begin{table}[t]
  \centering
  \caption{\textbf{Margin-based selective classification} (2K test samples).
    Margin(C/W): mean GroupSum margin for correct/wrong predictions.
    Separation: ratio of correct to wrong margin.
    Acc@$k$\%: accuracy when retaining only the $k$\% most confident samples.}
  \label{tab:margins}
  \small
  \begin{tabular}{@{}lcccccc@{}}
    \toprule
    Model & Circ Acc & Margin(C) & Margin(W) & Sep. & Acc@90\% & Acc@50\% \\
    \midrule
    \dlgn{}-\scsmall{}  & 48.9\% & 37.1 & 19.5 & 1.90$\times$ & 52.0\% & 64.3\% \\
    \tlgn{}-\scsmall{}  & 42.6\% & 38.4 & 22.9 & 1.68$\times$ & 44.8\% & 54.5\% \\
    \addlinespace
    \dlgn{}-\sclarge{}  & 51.7\% & 49.3 & 26.6 & 1.85$\times$ & 54.9\% & 66.1\% \\
    \tlgn{}-\sclarge{}  & 38.5\% & 66.3 & 41.3 & 1.61$\times$ & 39.9\% & 48.3\% \\
    \addlinespace
    \dlgn{}-\scvlarge{} & 52.0\% & 58.6 & 30.1 & 1.95$\times$ & 54.6\% & 66.3\% \\
    \tlgn{}-\scvlarge{} & 39.2\% & 85.6 & 51.9 & 1.65$\times$ & 41.2\% & 49.5\% \\
    \addlinespace
    \dlgn{}-\schuge{}   & 53.0\% & 106.4& 50.5 & 2.11$\times$ & 55.7\% & 70.3\% \\
    \tlgn{}-\schuge{}   & 48.4\% & 131.7& 73.9 & 1.78$\times$ & 50.6\% & 61.5\% \\
    \bottomrule
  \end{tabular}
\end{table}

Ternary margins are larger in absolute magnitude ($\{-1, 0, 1\}$ outputs
actively suppress competing classes), but \dlgn{} achieves better margin
\emph{separation}: 2.11$\times$ vs.\ 1.78$\times$ at \schuge{} scale, yielding
Acc@50\% of 70.3\% vs.\ 61.5\%.
Ternary separation improves with scale (1.65$\times$ $\to$ 1.78$\times$
from \scvlarge{} to \schuge{}; Acc@50\%: 49.5\% $\to$ 61.5\%), indicating
genuine improvement in discriminative quality of surviving neurons.

\subsection{Appendix: Supporting Experiments}
\label{sec:td-appendix}

\subsubsection{Decision boundary gallery}
\label{sec:td-boundaries}

Figure~\ref{fig:decision-boundaries} presents the full $5 \times 4$ decision
boundary gallery. Each row shows one dataset; columns display the raw data,
binary circuit boundary, ternary circuit boundary, and UNKNOWN density overlay.
Across all five datasets, the ternary UNKNOWN regions (orange) concentrate at
decision boundaries, confirming the pattern described in
Section~\ref{sec:td-contrast}.

\begin{figure}[htbp]
  \centering
  \includegraphics[width=\textwidth]{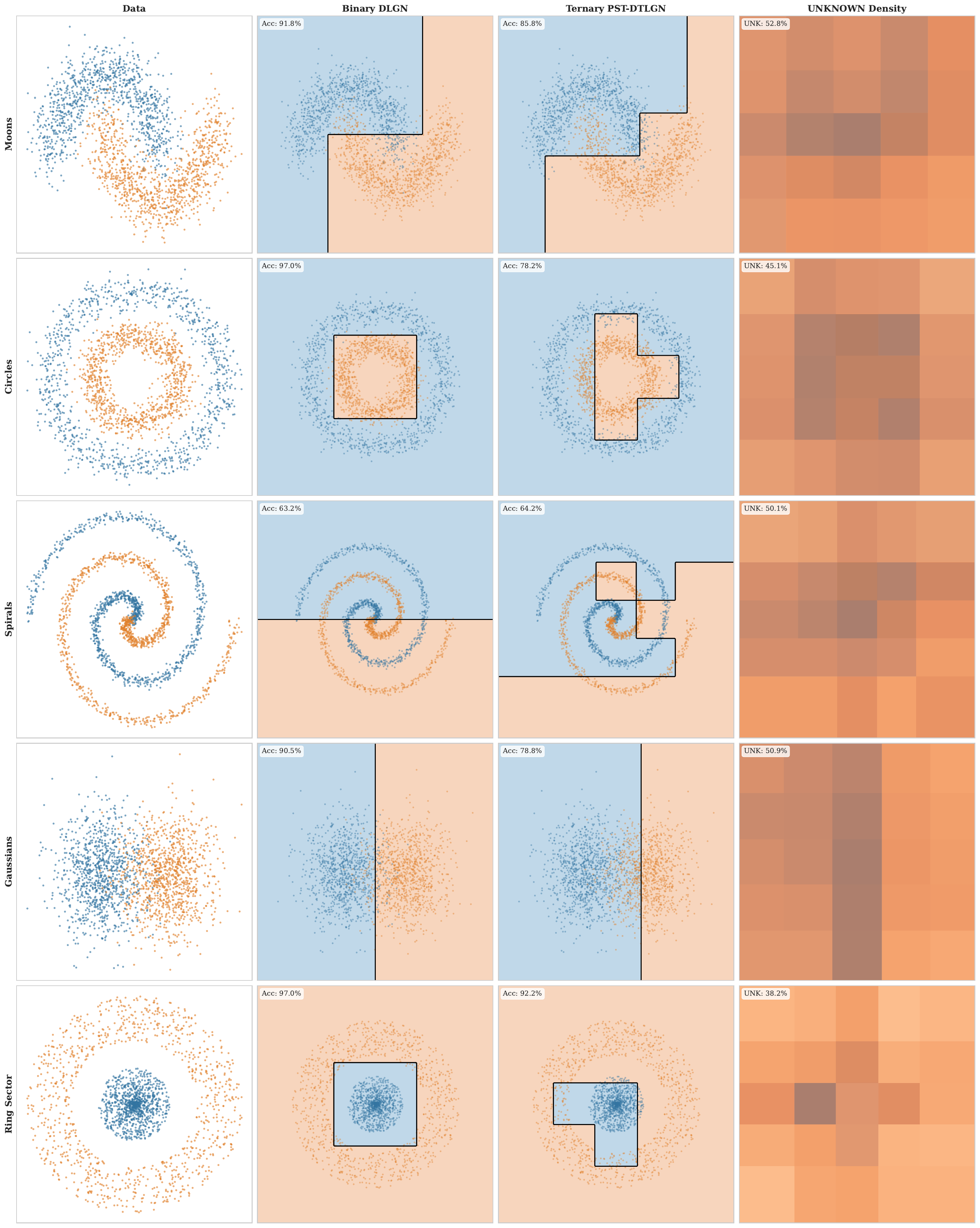}
  \caption{\textbf{Decision boundary gallery} ($5 \times 4$ grid).
    Rows: Moons, Circles, Spirals, Gaussians, Ring Sector.
    Columns: raw data, binary boundary, ternary boundary, UNKNOWN density.
    Orange regions mark where $>50\%$ of output neurons produce zero.}
  \label{fig:decision-boundaries}
\end{figure}

The UNKNOWN fraction ranges from 38.3\% (Ring Sector) to 52.9\% (Moons),
scaling with boundary complexity. Ring Sector, with its clean angular
separation, produces the smallest UNKNOWN band; Moons, with its high boundary
curvature and noise${}=0.5$, produces the largest.

\begin{table}[t]
  \centering
  \caption{\textbf{Gaussian separation sweep.}
    As separation increases, UNKNOWN fraction tracks the Bayes error.
    Binary accuracy matches the Bayes rate at all separations, confirming
    neither architecture is capacity-limited.}
  \label{tab:separation-sweep}
  \small
  \begin{tabular}{@{}rcccc@{}}
    \toprule
    Separation & Tern Acc & UNK\% & Bin Acc & Bayes \\
    \midrule
    $0.5\sigma$ & 59.8\% & 54.9\% & 68.2\% & 68.2\% \\
    $1.0\sigma$ & 71.0\% & 53.1\% & 83.5\% & 83.5\% \\
    $1.5\sigma$ & 78.7\% & 50.6\% & 90.5\% & 90.5\% \\
    $2.0\sigma$ & 85.5\% & 44.0\% & 97.8\% & 97.8\% \\
    $2.5\sigma$ & 92.2\% & 39.1\% & 99.0\% & 99.0\% \\
    $3.0\sigma$ & 93.5\% & 37.5\% & 99.8\% & 99.8\% \\
    \bottomrule
  \end{tabular}
\end{table}

\subsubsection{Accuracy vs.\ coverage curves}
\label{sec:td-coverage}

Figure~\ref{fig:accuracy-coverage} shows the full accuracy-vs-coverage curves
for three representative datasets. The monotonic increase confirms that the
GroupSum margin is a well-calibrated confidence proxy: low-margin predictions
are disproportionately incorrect, and abstaining on them systematically
improves accuracy on the retained set.

\begin{figure}[t]
  \centering
  \includegraphics[width=\textwidth]{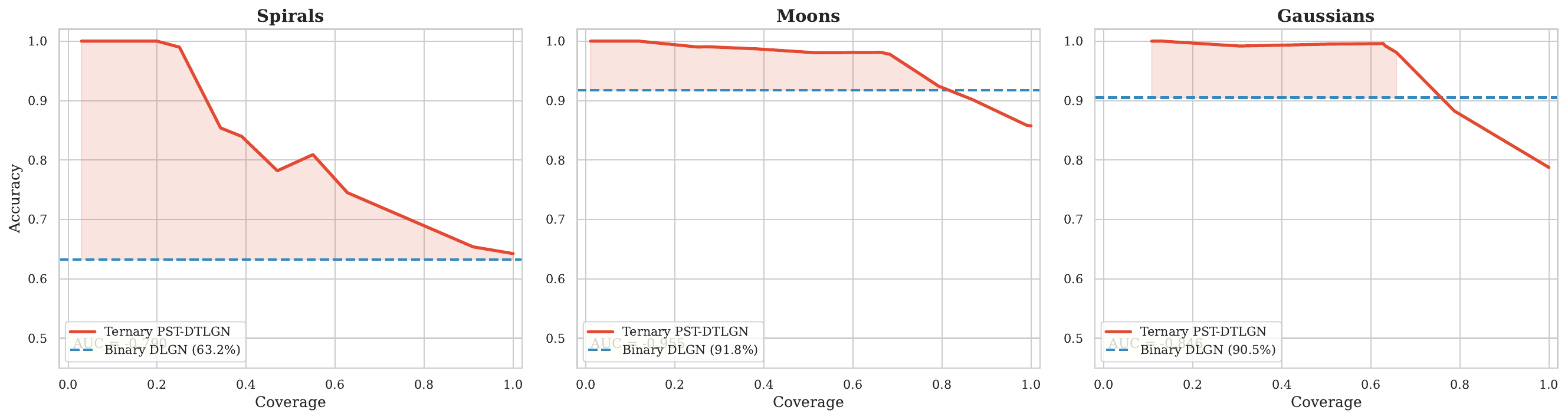}
  \caption{\textbf{Accuracy vs.\ coverage} (Spirals, Moons, Gaussians).
    As the circuit abstains on the lowest-margin samples, accuracy on the
    retained samples rises monotonically. The AUC quantifies selective
    prediction quality (closer to $-1.0$ is better).}
  \label{fig:accuracy-coverage}
\end{figure}

\subsubsection{Controlling the UNKNOWN band}
\label{sec:td-asymmetric}

The \texttt{delta\_fraction} parameter ($\delta \in [0, 1]$) controls the width
of the UNKNOWN encoding band. Table~\ref{tab:asymmetric} and
Figure~\ref{fig:asymmetric} show that peak accuracy on Moons occurs at
$\delta{=}0.50$ (95.0\%), not at either extreme. Notably, UNK\% does not reach
zero even at $\delta{=}0.0$ (42.0\%), because UNKNOWN outputs emerge from
internal truth tables, not only from the input encoding.

\begin{figure}[t]
  \centering
  \includegraphics[width=\textwidth]{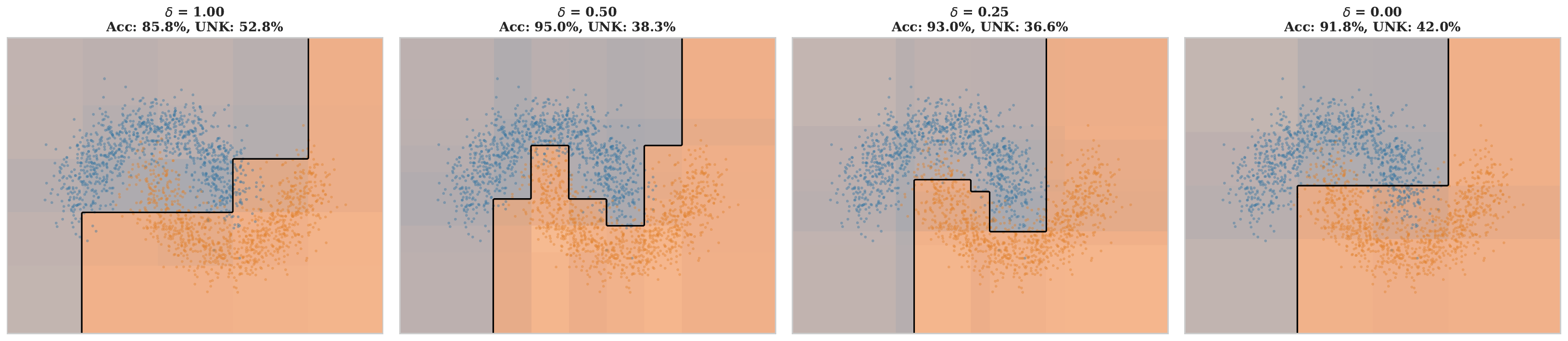}
  \caption{\textbf{Asymmetric thresholding} (Moons, four $\delta$ values).
    Peak accuracy (95.0\%) at $\delta{=}0.50$, not at the binary-like extreme.}
  \label{fig:asymmetric}
\end{figure}

\begin{table}[t]
  \centering
  \caption{\textbf{Asymmetric thresholding sweep} (Moons, test set).}
  \label{tab:asymmetric}
  \small
  \begin{tabular}{@{}rccc@{}}
    \toprule
    $\delta$ & Test Acc & UNK\% & Gap \\
    \midrule
    1.00 & 85.8\% & 52.8\% & 0.00\% \\
    0.50 & 95.0\% & 38.3\% & 0.00\% \\
    0.25 & 93.0\% & 36.6\% & 0.00\% \\
    0.00 & 91.7\% & 42.0\% & $-0.81$\% \\
    \bottomrule
  \end{tabular}
\end{table}

\subsubsection{Encoding resolution}
\label{sec:td-resolution}

Higher resolution simultaneously increases accuracy and decreases UNKNOWN rate
(Table~\ref{tab:resolution}, Figure~\ref{fig:resolution}): from 84.0\%
accuracy / 53.8\% UNK at $K{=}2$ to 95.5\% / 23.8\% at $K{=}16$, confirming
that the UNKNOWN band reflects input-level ambiguity resolvable by finer
quantization.

\begin{figure}[t]
  \centering
  \includegraphics[width=\textwidth]{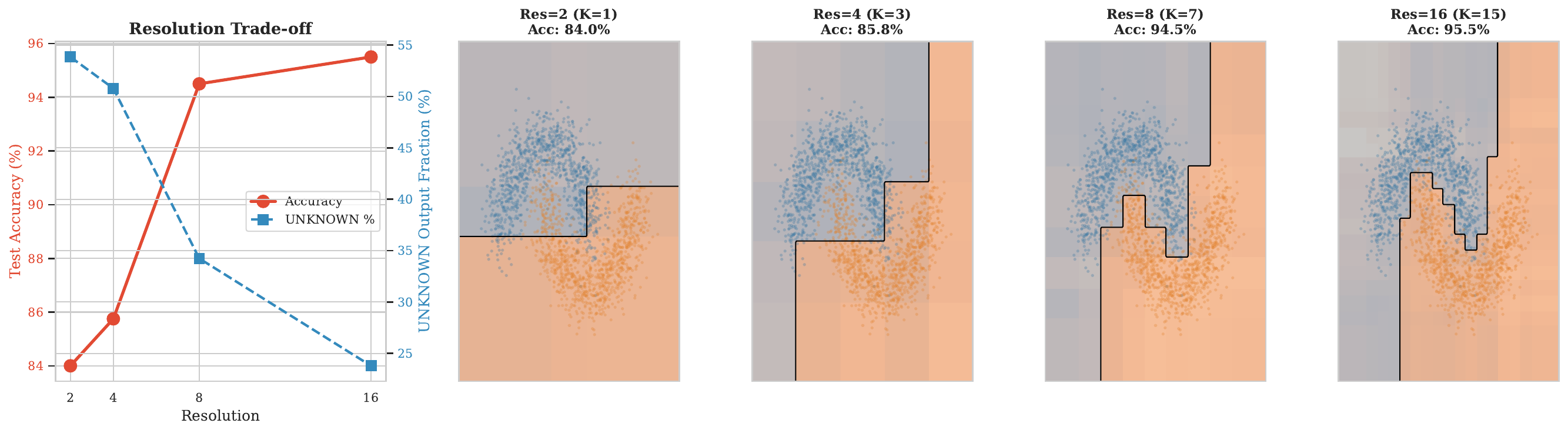}
  \caption{\textbf{Resolution vs.\ accuracy and UNKNOWN rate} (Moons).
    Higher resolution increases accuracy and decreases UNKNOWN fraction.}
  \label{fig:resolution}
\end{figure}

\begin{table}[t]
  \centering
  \caption{\textbf{Resolution sweep} (Moons, test set).}
  \label{tab:resolution}
  \small
  \begin{tabular}{@{}rcccc@{}}
    \toprule
    $K$ & Input Dim & Body Widths & Test Acc & UNK\% \\
    \midrule
     2 &  2 & $[128]^3$  & 84.0\% & 53.8\% \\
     4 &  6 & $[256]^3$  & 85.8\% & 50.8\% \\
     8 & 14 & $[512]^3$  & 94.5\% & 34.2\% \\
    16 & 30 & $[1024]^3$ & 95.5\% & 23.8\% \\
    \bottomrule
  \end{tabular}
\end{table}

\subsubsection{Gate diversity}
\label{sec:td-diversity}

\pst{} discovers ${\sim}1{,}200$ unique gates per dataset (out of 19,683), with
effective diversity ${\sim}1{,}050$, representing $70\times$ greater vocabulary
utilization than binary's fixed 16 (Table~\ref{tab:gate-diversity}).
Ternary redundancy is 28-31\% vs.\ binary's 99.1\%, reflecting the two
strategies: binary achieves expressiveness through dense repetition; ternary
through gate specialization.

\begin{table}[t]
  \centering
  \caption{\textbf{Gate diversity} (both architectures, all datasets).}
  \label{tab:gate-diversity}
  \small
  \begin{tabular}{@{}lrrc|rrc@{}}
    \toprule
    & \multicolumn{3}{c|}{Ternary} & \multicolumn{3}{c}{Binary} \\
    Dataset & Unique & Eff.\ Div & Red.\% & Unique & Eff.\ Div & Red.\% \\
    \midrule
    Moons       & 1,214 & 1,051.8 & 30.1\% & 16 & 15.3 & 99.1\% \\
    Circles     & 1,225 & 1,065.1 & 29.4\% & 16 & 15.6 & 99.1\% \\
    Spirals     & 1,224 & 1,062.2 & 29.5\% & 16 & 15.3 & 99.1\% \\
    Gaussians   & 1,201 & 1,037.9 & 30.8\% & 16 & 15.6 & 99.1\% \\
    Ring Sector & 1,253 & 1,089.1 & 27.8\% & 16 & 15.8 & 99.1\% \\
    \bottomrule
  \end{tabular}
\end{table}

\subsubsection{Hardening gap comparison}
\label{sec:td-hardening-gap}

Figure~\ref{fig:hardening-gap} compares the hardening gap across all datasets.
Both architectures achieve small gaps under discrete-input training (binary
$\leq 3.19$\%, ternary $\leq 8.31$\%). The near-zero ternary gaps on Moons
(0.00\%), Gaussians (0.31\%), and Ring Sector (0.38\%) validate the fidelity
of the \pst{} hardening procedure.

\begin{figure}[t]
  \centering
  \includegraphics[width=0.75\textwidth]{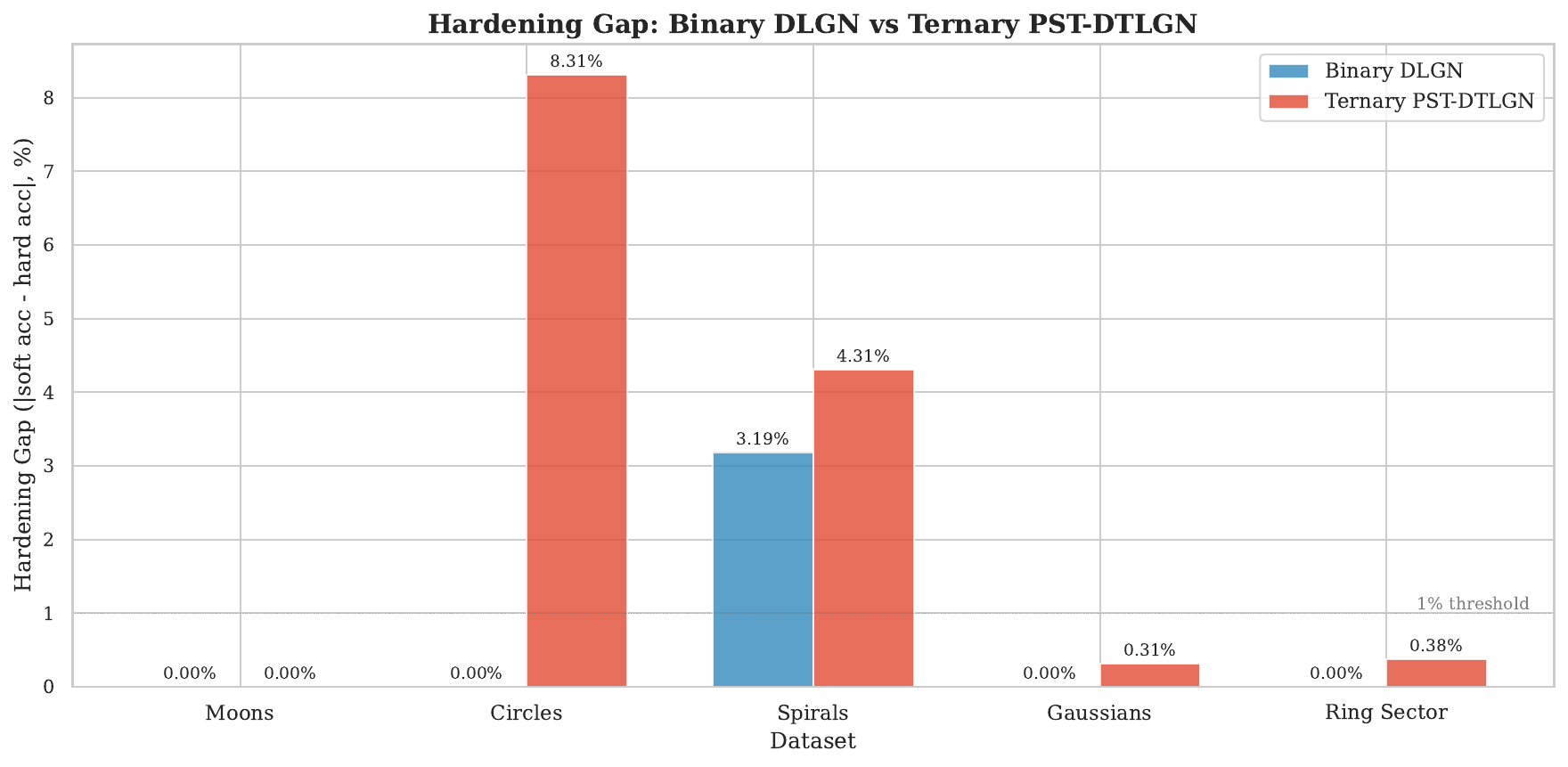}
  \caption{\textbf{Hardening gap} (all datasets). Both architectures achieve
    small gaps. The ternary gap is elevated on Circles and Spirals.}
  \label{fig:hardening-gap}
\end{figure}

  \section{Polynomial Representations for Boolean Logic Functions}
\label{appendix:polynomial_algebra}

Let $P_n$ denote the set of real multi-linear polynomials that represent Boolean functions over $n$ variables. Given real multi-linear polynomials $p,\,q\in P_n$, define the binary operators meet ($\wedge$) and join ($\vee$):
\[
    p\wedge q:=-\dfrac{1}{2}+\dfrac{1}{2}p(x)+\dfrac{1}{2}q(x)+\dfrac{1}{2}p(x)q(x),\qquad p\vee q:=\dfrac{1}{2}+\dfrac{1}{2}p(x)+\dfrac{1}{2}q(x)-\dfrac{1}{2}p(x)q(x),
\]
and define the unary operator complementation ($\neg$):
\[
    \neg p:=-p(x).
\]
Then $(P_n,\,\vee,\,\wedge,\,\neg,\,\bot,\,\top)$ is a Boolean algebra with the top element being the constant polynomial $p(x)=1$ and the bottom element being the constant polynomial $p(x)=-1$. For two-input Boolean functions ($n=2$), the Boolean algebra of real multi-linear polynomials is generated by the set $\{x_1,\,x_2\}$, meaning one can reconstruct all Boolean functions using the operations meet, join, and complement applied to the set of generators.

\subsection{The Boolean Group}
The exclusive not-or\footnote{Typically a Boolean group derived from a Boolean algebra uses symmetric difference (exclusive or) for group multiplication. If we instead defined \TRUE~as $-1$ and \FALSE~as $+1$ (intepreted as the square roots of unity), then XOR would be the intuitive choice satisfying $p\oplus q:=p(x)q(x)$.} (XNOR) of elements in the Boolean algebra can be defined as:
\[
    p\odot q:=p(x)q(x).
\]
The structure $(P_n,\,\odot,\,\top)$ is a finite abelian group, with the constant polynomial $p(x)=1$ serving as the identity element and each element acting as its own inverse, $p\odot p=\top$. This subtle notion allows us to employ the results of Fourier analysis on finite groups (see: \cite{terras1999fourier}). We restrict the remaining analysis to the case where $n=2$. Define the character homomorphism $\mathcal{X}:P_2\to\mathbb{R}$ as the trivial mapping $p\mapsto p(x)$ for $p\in P_2$, $x\in{\{-1,\,1\}}^2$. Characters satisfy the condition $\mathcal{X}(p\odot q)=\mathcal{X}(p)\mathcal{X}(q)$ for all $p,\,q\in P_2$. Define an inner product on the space of characters as:
\begin{equation}
    \langle\varphi,\,\psi\rangle=\dfrac{1}{4}\sum_{x\in{\{-1,\,1\}}^2}\varphi(x)\psi(x).\label{eq:boolean_ip}
\end{equation}
We are interested in identifying orthonormal functions with respect to \eqref{eq:boolean_ip} that form the basis for all representations of Boolean functions. Consider the character functions $\{1,\,x_1,\,x_2,\,x_1x_2\}\subset \mathcal{X}(P_2)$. These functions form an orthonormal basis for the characters of $P_2$. Enumerating the basis functions such that:
\[
    \varphi_0(x)=1,\quad\varphi_1(x)=x_1,\quad\varphi_2(x)=x_2,\quad\varphi_3(x)=x_1x_2,
\]
we have:
\[
    \langle \varphi_i,\,\varphi_j\rangle=\begin{cases}
        \;1,\quad& i=j, \\
        \;0,\quad&i\neq j.
    \end{cases}
\]
Furthermore, any $p\in P_2$, $\mathcal{X}(p)$ can be uniquely represented as a linear combination of these basis functions. Given $p\in P_2$, the Fourier coefficients are determined by:
\[
    w_i=\langle\mathcal{X}(p),\,\varphi_i\rangle,
\]
and we have then:
\[
    p(x)=\sum_{i=1}^{4}w_i\varphi_i(x).
\]
\end{document}